\documentclass{article}

\usepackage{microtype}
\usepackage{graphicx}
\usepackage{subfigure}
\usepackage{booktabs} % for professional tables

\usepackage{hyperref}

\usepackage[accepted]{icml2023}

\usepackage{amsmath}
\usepackage{amssymb}
\usepackage{mathtools}
\usepackage{amsthm}
\usepackage{mathrsfs}
\usepackage{color}
\usepackage[capitalize,noabbrev]{cleveref}

\theoremstyle{plain}
\newtheorem{theorem}{Theorem}[section]
\newtheorem{proposition}[theorem]{Proposition}
\newtheorem{lemma}[theorem]{Lemma}

\theoremstyle{definition}
\newtheorem{definition}[theorem]{Definition}

\theoremstyle{remark}

\usepackage[textsize=tiny]{todonotes}

\icmltitlerunning{Variance Control for Distributional Reinforcement Learning}

\begin{document}

\twocolumn[
\icmltitle{Variance Control for Distributional Reinforcement Learning}



\icmlsetsymbol{equal}{*}

\begin{icmlauthorlist}
\icmlauthor{Qi Kuang}{equal,sufe}
\icmlauthor{Zhoufan Zhu}{equal,sufe}
\icmlauthor{Liwen Zhang}{sufe}
\icmlauthor{Fan Zhou}{sufe}

%\icmlauthor{Firstname8 Lastname8}{sch}
%\icmlauthor{Firstname8 Lastname8}{yyy,comp}
\end{icmlauthorlist}

\icmlaffiliation{sufe}{School of Statistics and Management, Shanghai University of Finance and Economics, Shanghai, China}
%\icmlaffiliation{comp}{Company Name, Location, Country}
%\icmlaffiliation{sch}{School of ZZZ, Institute of WWW, Location, Country}

%\icmlcorrespondingauthor{Firstname1 Lastname1}{first1.last1@xxx.edu}
\icmlcorrespondingauthor{Fan Zhou}{zhoufan@mail.shufe.edu.cn}

\icmlkeywords{Machine Learning, ICML}

\vskip 0.3in
]



\printAffiliationsAndNotice{\icmlEqualContribution} % otherwise use the standard text.

\begin{abstract}

Although distributional reinforcement learning (DRL) has been widely examined in the past few years, very few studies investigate the validity of the obtained Q-function estimator in the distributional setting. To fully understand how the approximation errors of the Q-function affect the whole training process, we do some error analysis and theoretically show how to reduce both the bias and the variance of the error terms. With this new understanding, we construct a new estimator \emph{Quantiled Expansion Mean} (QEM) and introduce a new DRL algorithm (QEMRL) from the statistical perspective. We extensively evaluate our QEMRL algorithm on a variety of Atari and Mujoco benchmark tasks and demonstrate that QEMRL achieves significant improvement over baseline algorithms in terms of sample efficiency and convergence performance.
\end{abstract}

\section{Introduction}
\label{intro}

Distributional Reinforcement Learning (DRL) algorithms have been shown to achieve state-of-art performance in RL benchmark tasks \cite{C51,QRDQN,IQN,FQF,NC-QRDQN,NDQFN}.  The core idea of DRL is to estimate the entire distribution of the future return instead of its expectation value, i.e. the Q-function, which captures the intrinsic uncertainty of the whole process in three folds:
%single expectation value in classical RL. 
%This improvement helps to capture the intrinsic uncertainty in three folds: 
(i) the stochasticity of rewards, (ii) the indeterminacy of the policy, %for interacting with the environment, 
and (iii) the inherent randomness of transition dynamics.  Existing DRL algorithms parameterize the return distribution in different ways, including categorical return atoms \cite{C51}, expectiles \cite{EDRL}, particles \cite{MMDRL}, and quantiles \cite{QRDQN,IQN}.
%and achieve impressive practical performances in real-world tasks. 
Among these works, the quantile-based algorithm is widely used due to its simplicity, efficiency of training, and flexibility in modeling the return distribution.

Although the existing quantile-based algorithms achieve remarkable empirical success, the approximated distribution  
%approximated quantile distribution 
still requires further understanding and investigation.  One aspect is the crossing issue, namely, a violation of the monotonicity of the 
obtained quantile estimations. \citet{NC-QRDQN,NDQFN} solves this issue by enforcing the monotonicity of the estimated quantiles using some well-designed neural networks. However, these methods may suffer from some underestimation or overestimation issues. In other words, the estimated quantiles tend to be higher or lower than their true values. Considering this shortcoming, \citet{SPL-DQN} applies monotonic rational-quadratic splines to ensure monotonicity, but their algorithm is computationally expensive and hard to implement in large-scale tasks. 

Another aspect is regard to the tail behavior of the return distribution. It is widely acknowledged that the precision of tail estimation highly depends on the frequency of tail observations \cite{Koenker2005Quantile}. Due to data sparsity, the quantile estimation is often unstable at the tails. To alleviate this instability, 
%caused by tail estimation, 
\citet{TQC} proposes to truncate the right tail of the approximated return distribution by discarding some topmost atoms. However, this approach lacks theoretical support and ignores the potentially useful information hidden in the tail.%To utilize tail behavior for the efficient exploration, \citet{DLTV} use the upper tail variability of distribution as an exploration bonus in light of its more optimism in face of uncertainty than the low tail. The aforementioned improvements are guided by experimental results, however, there is no clear theoretical or intuitive explanation for them.

The crossing issue and tail unrealization illustrate that there is a substantial gap between the quantile estimation and its true value. This finding reduces the reliability of the Q-function estimator obtained by quantile-based algorithms and inspires us to further minimize the difference between the estimated Q-function and its true value. In particular, the error associated with Q-function approximation can be decomposed into three parts:
\begin{align}\label{decomposition} \nonumber
\Delta &\equiv Q^{\pi}_{\theta}(x,a) - Q^{\pi}(x,a) = \mathbb{E}Z^{\pi}_{\theta}(x,a) - \mathbb{E}Z^{\pi}(x,a)
\\ \nonumber
&=\underbrace{\mathbb{E}Z^{\pi}_{\theta}(x,a)-\mathbb{E}_{x'\sim \mathcal{D}}[R+\gamma Z^{\pi}_{\theta}(x',a')]}_
{\text { Target Approximation Error}~~\mathcal{E}_1} \nonumber \\
&\quad + \underbrace{ \mathbb{E}_{x'\sim \mathcal{D}}[R+\gamma Z^{\pi}_{\theta}(x',a')] -\mathbb{E}_{x'\sim P}[R+\gamma Z^{\pi}_{\theta}(x',a')] }_{\text { Bellman operator Approximation Error}~~\mathcal{E}_2}\nonumber \\
&\quad + \underbrace{ \mathbb{E}_{x'\sim P}[R+\gamma Z^{\pi}_{\theta}(x',a')] -\mathbb{E} Z^{\pi}(x,a) }_{\text {Parametrization Induced Error}~~\mathcal{E}_3},
\end{align}
%\begin{align*}
%\Delta &\equiv Q^{\pi}_{\theta}(x,a) - Q^{\pi}(x,a) \\
%&= \mathbb{E}Z^{\pi}_{\theta}(x,a) - \mathbb{E}Z^{\pi}(x,a) \\
%&=\underbrace{\mathbb{E}Z^{\pi}_{\theta}(x,a)-\mathbb{E}_{x'\sim \mathcal{D}}[R+\gamma Z^{\pi}_{\theta}(x',a')]}_
%{\text { Target Approximation Error}~~\mathcal{E}_1}  \\
%&\quad + \underbrace{ \mathbb{E}_{x'\sim \mathcal{D}}[R+\gamma Z^{\pi}_{\theta}(x',a')] -\mathbb{E}_{x'\sim P}[R+\gamma Z^{\pi}_{\theta}(x',a')] }_{\text { Bellman operator Approximation Error}~~\mathcal{E}_2} %\\
%&\quad + \underbrace{ \mathbb{E}_{x'\sim P}[R+\gamma Z^{\pi}_{\theta}(x',a')] -\mathbb{E} Z^{\pi}(x,a) }_{\text {Parametrization Induced Error}~~\mathcal{E}_3},
%\end{align*}
%\begin{align}\label{decomposition} 
%&\quad + \underbrace{ \mathbb{E}_{x'\sim P}[R+\gamma Z^{\pi}_{\theta}(x',a')] -%\mathbb{E} Z^{\pi}(x,a) }_{\text {Parametrization Induced Error}~~\mathcal{E}_3},
%\end{align}
where $Q^{\pi}(\cdot)$ is the true Q-function, $Q^{\pi}_{\theta}(\cdot)$ is the approximated Q-function, $Z^{\pi}$ is the random variable with the true return distribution, $Z_{\theta}^{\pi}$ is the random variable with the approximated quantile function parameterized by a set of quantiles $\theta$, $\mathcal{D}$ is the replay buffer, and $P$ is the transition kernel.
These errors can be attributed to different kinds of approximations in DRL \cite{rowland2018analysis}, including (i) parameterization and its associated projection operators, (ii) stochastic approximation of the Bellman operator, and (iii) gradient updates through quantile loss. 
 
 We elaborate on the properties of the three error terms in (\ref{decomposition}). $\mathcal{E}_1$ is derived from the target approximation in quantile loss. $\mathcal{E}_2$ is caused by the stochastic approximation of the Bellman operator. $\mathcal{E}_3$ results from the parametrization of quantiles and the corresponding projection operator. Among the three, $\mathcal{E}_3$ can be theoretically  eliminated if the representation size is large enough, whereas $\mathcal{E}_1 + \mathcal{E}_2$  is inevitable in practice 
 %as a result of the sample-based update 
 due to the batch-based optimization procedure. Therefore, controlling the variance $\mathrm{Var}(\mathcal{E}_1 + \mathcal{E}_2)$ can significantly speed up the training convergence (see an illustrating example in \cref{illustration}). Thus, one main target of this work is to reduce the two inevitable errors $\mathcal{E}_1$ and $\mathcal{E}_2$, and subsequently improve the existing DRL algorithms. 
 
 %More specifically, $\mathcal{E}_1$ and $\mathcal{E}_2$ can be interpreted as a result of the sample-based update associated with (ii) and (iii), and $\mathcal{E}_3$ can be interpreted according to (i). 

 %\cref{illustration} illustrates how errors vary with updates. The parametrization-induced error $\mathcal{E}_3$ (grey areas) can be theoretically eliminated when the size of the representation, i.e., the number of quantiles, is large enough. The approximation errors $\mathcal{E}_1$ and $\mathcal{E}_2$ (blue areas) which determine the variance of the Q-function, will slow down convergence during the training process. To decrease these errors quickly,  it is necessary to apply the variance reduction technique QCM estimator to control these errors.

\begin{figure}[!ht]
%\vskip -0.1in
\includegraphics[width=0.999\linewidth]{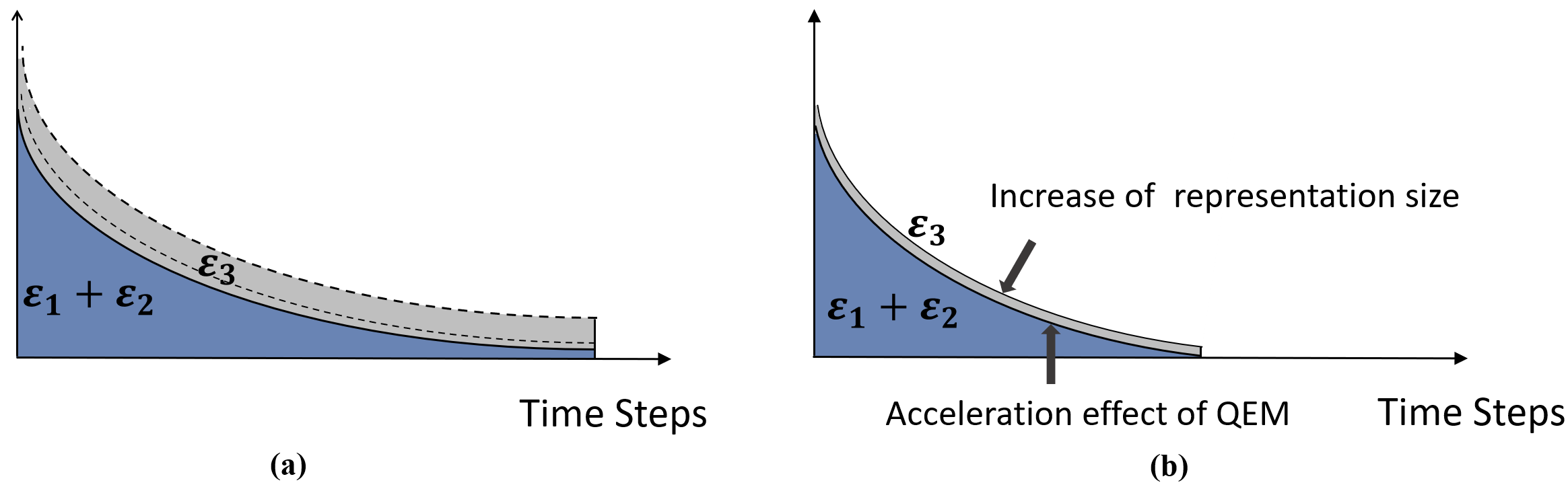}
\vskip -0.05in
\caption{Error decay during training. (a) The parameterization-induced error $\mathcal{E}_3$ (grey areas) remains constant over time with a fixed representation size. The approximation errors $\mathcal{E}_1$ and $\mathcal{E}_2$ (blue areas) decrease slowly with time steps. (b) Increase the size of the representation (i.e., the number of quantiles), $\mathcal{E}_3$ can be theoretically eliminated. By applying the variance reduction technique QEM estimator, $\mathcal{E}_1+\mathcal{E}_2$  can be quickly decreased, resulting in faster convergence of algorithms. }
\label{illustration}
\vskip -0.05in
\end{figure}

The contributions of this work are summarized as follows, 
\begin{itemize}
\setlength{\itemsep}{0pt}
\item We offer a rigorous investigation on the three error terms $\mathcal{E}_1$, $\mathcal{E}_2$, and $\mathcal{E}_3$ in DRL, and find that the approximation errors result from the heteroskedasticity of quantile estimates, especially tail estimates.
\item  We borrow the idea from the Cornish-Fisher Expansion \cite{Fisher1938}, and propose a statistically robust DRL algorithm, called QEMRL, to reduce the variance of the estimated Q-function. 
\item We show that QEMRL achieves a higher stability and a faster convergence rate from both theoretical and empirical perspectives. 
%provide theoretical and experimental results to show the QCMRL leads to stability and faster convergence.
\end{itemize}

%%%%%%%%%%%%%%%%%%%%%%%%%%%%%%%%%%%%%%%%%%%%%%%%%%%%%%%%%%%%%%%%%%%%%%%%%%%%%%%%%%%%%%%%%%%%%%
%%%%%%%%%%%%%%%%%%%                    SECTION2                %%%%%%%%%%%%%%%%%%%%%%%%%%%%%
%%%%%%%%%%%%%%%%%%%%%%%%%%%%%%%%%%%%%%%%%%%%%%%%%%%%%%%%%%%%%%%%%%%%%%%%%%%%%%%%%%%%%%%%%%%%%

\section{Background}
\subsection{Reinforcement Learning}
Consider a finite Markov Decision Process (MDP) $(\mathcal{X}, \mathcal{A}, P, \gamma, \mathcal{R})$, with a finite set of states $\mathcal{X}$, a finite set of actions $\mathcal{A}$, the transition kernel $P :\mathcal{X} \times \mathcal{A} \rightarrow \mathscr{P}(\mathcal{X})$, the discounted factor $\gamma \in [0,1)$, and the bounded reward function $\mathcal{R}:\mathcal{X} \times \mathcal{A} \rightarrow \mathscr{P}([- R_{max}, R_{max}])$. At each timestep, an agent observes state $X_{t} \in \mathcal{X}$, takes an action $A_{t} \in \mathcal{A}$, transfers to the next state $X_{t+1} \sim P\left(\cdot \mid X_{t}, A_{t}\right)$, and receives a reward $R_{t} \sim \mathcal{R}\left(X_{t}, A_{t}\right)$. The state-action value function $Q^{\pi}:\mathcal{X} \times \mathcal{A} \rightarrow \mathbb{R}$ of a policy $\pi:\mathcal{X}\rightarrow \mathscr{P}(\mathcal{A})$ is the expected discounted sum of rewards starting from $x$, taking an action $a$ and following a policy $\pi$. $\mathscr{P}(\mathcal{X})$ denotes the set of probability distributions on a space $\mathcal{X}$.

%For a given policy $\pi$, the return $Z^{\pi}(x, a)$ is a random variable representing cumulative rewards the agent obtains following the policy $\pi$, i.e. $Z^{\pi}(x, a)=\sum_{t=0}^{\infty} \gamma^{t} R\left(x_{t}, a_{t}\right)$ with $x_{0}=x, a_{0}=a$ and $x_{t+1} \sim P\left(\cdot \mid x_{t}, a_{t}\right), a_{t} \sim \pi\left(\cdot \mid s_{t}\right)$.

%There are two principal goals in reinforcement learning: \emph{evaluation} and \emph{control}. The \emph{evaluation} aims to compute the state-action value function $Q^{\pi}(x, a)$, given a fixed policy $\pi$, while the task of \emph{control} seeks to find an optimal policy $\pi^*$ to maximize the state-action value function $Q^{\pi}(x, a)$.

 %First, given a fixed policy $\pi$, the task of \emph{evaluation} of $\pi$ consists of computing the state-action value function $Q^{\pi}(x, a)=$ $\mathbb{E}_{\pi}\left[\sum_{t=0}^{\infty} \gamma^{t} R_{t} \mid X_{0}=x, A_{0}=a\right]$, where $\mathbb{E}_{\pi}$ indicates that at each time step $t \in \mathbb{N}$, the agent's action is sampled from $\pi$. Second, the task of \emph{control} consists of finding an optimal policy $\pi^*$ to maximize the state-action value function $Q^{\pi}(x, a)$.

%\textbf{Expected RL}
%In the evaluation task, the 
The classic Bellman equation \cite{bellman1966dynamic} relates expected return at each state-action pair $(x, a)$ to the expected returns at possible next states by:
\begin{small}
\begin{align}\label{eq1}
Q^{\pi}(x, a)=\mathbb{E}_{\pi}\left[R_{0}+\gamma Q^{\pi}\left(X_{1}, A_{1}\right) \mid X_{0}=x, A_{0}=a\right].
\end{align}
\end{small}
In the learning task, Q-Learning~\cite{watkins1989learning} employs a common way to obtain $\pi^*$, which is to find the unique fixed point $Q^* = Q^{\pi^*}$ of the Bellman optimality equation:
\begin{small}
\begin{align*}
Q^{*}(x, a)=\mathbb{E}\left[R_{0}+\gamma\underset{a'\in\mathcal{A}}{\max} Q^{*}\left(X_1, a'\right) \mid X_{0}=x, A_{0}=a\right].
\end{align*}
\end{small}

%In practice, the state-action value function $Q$ can be approximated by a neural network $Q_\theta$ \cite{DQN}. The networks parameters are iteratively updated by minimizing the squared temporal difference (TD) error
%\begin{small}
%\begin{eqnarray} \nonumber
%&& \delta^2 = \left[r+\gamma\max_{a' \in \mathcal{A}} Q_{\theta} (x', a') -Q_{\theta} (x, a) \right]^2,
%\end{eqnarray}
%\end{small}
%on a sampled transition $(x, a, r, x')$, collected by rXning an $\epsilon$-greedy policy over $Q_{\theta}X

%\begin{small}
%\begin{align}\nonumber
%\eta_{\pi}(x, a)=\operatorname{Law}_{\pi}\left(\sum_{t=0}^{\infty} \gamma^{t} R_{t} \mid X_{0}=x, A_{0}=a\right),
%\end{align}
%\end{small}\in\mathscr{P}(\mathbb{R})

\subsection{Distributional Reinforcement Learning}
Instead of directly estimating the expectation $Q^\pi(x,a)$, DRL focuses on estimating the distribution of the sum of discounted rewards $\eta_{\pi}(x, a) = \mathcal{D}(\sum_{t=0}^{\infty} \gamma^{t} R_{t} \mid X_{0}=x, A_{0}=a )$ to sufficiently capture the intrinsic randomness, where $\mathcal{D}$ extract the probability distribution of a random variable.
%which is the distribution of the sum of discounted rewards $\sum_{t=0}^{\infty} \gamma^{t} R_{t} \mid X_{0}=x, A_{0}=a$. 
In analogy with \cref{eq1}, 
%the return distribution 
$\eta_{\pi}$ satisfies the distributional Bellman equation \cite{C51} as follows,
\begin{align*}
\eta_{\pi}(x,a)=&\left(\mathcal{T}^{\pi} \eta_{\pi}\right){(x, a)} \\=& \mathbb{E}_{\pi}\left[(f_{\gamma, r})_{\#}\eta_{\pi}(X_1,A_1)\mid X_{0}=x, A_{0}=a\right]
\end{align*} 
%\int_{\mathbb{R}}\sum_{\mathcal{X}\times\mathcal{A}}(f_{r, \gamma})_{\#}\eta_{\pi}(x',a') \pi(a'|x') \mathcal{R}(\mathrm{~d}r| x, a)P(x'|x,a),
%\begin{align}\nonumber
%\mathcal{T}^{\pi} Z(x, a) &\stackrel{D}{:=}R(x, a)+\gamma Z\left(x^{\prime}, a^{\prime}\right), a^{\prime} \sim \pi\left(\cdot \mid x^{\prime}\right), \\
%\mathcal{T} Z(x, a) &\stackrel{D}{:=} R(x, a)+\gamma Z\left(x^{\prime}, \underset{a^{\prime} \in \mathcal{A}}{\arg \max } \mathbb{E}_{P} Z\left(x^{\prime}, a^{\prime}\right)\right),\nonumber 
%\end{align}
where $f_{\gamma, r}:\mathbb{R}\to\mathbb{R}$ is defined by $f_{\gamma, r}(x)=r+\gamma x,$ and $\left(f_{\gamma, r}\right)_{\#} \eta$ is the pushforward measure of $\eta$ by $f_{\gamma, r}$. Note that $\eta_{\pi}$ is the fixed point of distributional Bellman operator $\mathcal{T}^{\pi}:\mathscr{P}(\mathbb{R})^{\mathcal{X}\times\mathcal{A}}\to\mathscr{P}(\mathbb{R})^{\mathcal{X}\times\mathcal{A}}$, i.e., $\mathcal{T}^{\pi} \eta_{\pi}=\eta_{\pi}$. 
%Empirically, this distributional perspective has been shown to significantly improve the performance in the Atari benchmark.
%A more familiar form of distributional Bellman equation is to define the random variable $Z(x,a)$ distributed according to $\eta(x, a)$, and directly relate the distribution of the random return at successor state-action pairs $(x',a')$
%\begin{align*} 
%\mathcal{T}^{\pi} Z(x, a) &\stackrel{D}{:=}R(x, a)+\gamma Z(X',A')\\
%&\stackrel{D}{:=}R(x, a)+\gamma P^{\pi}Z(x^{\prime}, a^{\prime}).
%\end{align*}
%Where, $X'\sim P(\cdot|x,a)$ andinventor\pi(\cdot|X')$. However, the definition above is mathematically incomplete \cite{bdr2022}. This requires us to specify with measure-theoretic considerations.

In general, the return distribution supports a wide range of possible returns and its shape can be quite complex. Moreover, the transition dynamics are usually unknown in practice, and thus the full computation of the distributional Bellman operator is usually either impossible or computationally infeasible. In the following subsections, we review two main categories of DRL algorithms relying on parametric approximations and projection operators.

\subsubsection{Categorical distributional RL}
Categorical distributional RL (CDRL, \citealp{C51}) represents the return distribution $\eta$ with a categorical form $\eta(x,a)=\sum_{i=1}^{N} p_{i}(x,a) \delta_{z_{i}}$, where $\delta_{z}$ denotes the Dirac distribution at $z$. 
%The values 
$z_{1} \leq z_{2} \leq \ldots \leq z_{N}$ are evenly spaced locations, 
%with an even space, 
and $\left\{p_{i}\right\}_{i=1}^{N}$ are the corresponding probabilities learned using the Bellman update,
%following corresponding Bellman update form:
$$
\eta(x, a) \leftarrow\left(\Pi_{\mathcal{C}} \mathcal{T}^{\pi} \eta\right)(x, a),
$$
where $\Pi_{\mathcal{C}}:\mathscr{P}(\mathbb{R})\to\mathscr{P}(\{z_{1},z_{2} \ldots z_{N}\})$ is a {\it categorical projection} operator which ensures the return distribution supported only on $\left\{z_{1}, \ldots, z_{N}\right\}$. In practice, CDRL with $N=51$ has been shown to achieve significant improvement in Atari games.

\subsubsection{Quantiled distributional RL}
Quantiled distributional RL (QDRL, \citealp{QRDQN}) represents the return distribution with a mixture of Diracs $\eta(x,a)=\frac{1}{N} \sum_{i=1}^{N} \delta_{\theta_{i}(x,a)}$, where $\left\{\theta_{i}(x,a)\right\}_{i=1}^{N}$ are learnable parameters. The Bellman operator moves each atom location $\theta_{i}$ towards $\tau_{i}$-th quantile of the target distribution $\eta'(x,a):=\mathcal{T}^{\pi}\eta(x,a)$, where $\tau_{i}=\frac{2i-1}{2N}$. The corresponding Bellman update form is:
$$
\eta(x, a) \leftarrow\left(\Pi_{\mathcal{W}_{1}} \mathcal{T}^{\pi} \eta\right)(x, a),
$$
where $\Pi_{\mathcal{W}_{1}}:\mathscr{P}(\mathbb{R})\to\mathscr{P}(\mathbb{R})$ is a quantile projection operator defined by $\Pi_{\mathcal{W}_{1}}\mu=\frac{1}{N} \sum_{i=1}^{N} \delta_{F^{-1}_{\mu}(\tau_{i})}$, and $F_{\mu}$ is the cumulative distribution function (CDF) of  $\mu$. $F_{\eta'}^{-1}(\tau)$ can be characterized as the minimizer of the quantile regression loss, while the atom locations $\theta$ can be updated by minimizing the following loss function
%following the gradient of the loss:
\begin{small}
\begin{align}\label{eqqr}
\mathcal{L}_{QR}(\theta;\eta',\tau)=\mathbb{E}_{Z \sim \eta'}\left(\left[\tau \textbf{1}_{Z>\theta}+\left(1-\tau\right) \textbf{1}_{Z \leq \theta}\right]|Z-\theta|\right).
\end{align}
\end{small}

%%%%%%%%%%%%%%%%%%%%%%%%%%%%%%%%%%%%%%%%%%%%%%%%%%%%%%%%%%%%%%%%%%%%%%%%%%%%%%%
 %                         SECTION  3
%%%%%%%%%%%%%%%%%%%%%%%%%%%%%%%%%%%%%%%%%%%%%%%%%%%%%%%%%%%%%%%%%%%%%%%%%%%%%%%

%\section{Analysis of Error Decomposition of Q-function in Distributional RL}\label{sec3}
\section{Error Analysis of Distributional RL}\label{sec3}

%In this section, we further study the three error terms $\mathcal{E}_1$, $\mathcal{E}_2$ and $\mathcal{E}_3$ in \cref{decomposition}. 
As mentioned in Section \ref{intro}, the parametrization induced error $\mathcal{E}_3$ in \cref{decomposition} comes from quantile representation and its projection operator, which can be eliminated as $N \rightarrow \infty$. 
%thus leading to the reduction of $\Delta$. 
However, as illustrated in \cref{illustration}, the approximation errors $\mathcal{E}_1$ and $\mathcal{E}_2$ are unavoidable in practice and a high variance $\mathrm{Var}(\mathcal{E}_1 + \mathcal{E}_2)$ may lead to unstable performance of DRL algorithms. %Therefore, it is important to control the variance $\mathrm{Var}(\mathcal{E}_1 + \mathcal{E}_2)$ in practice.
Thus, in this section, we further study the three error terms $\mathcal{E}_1$, $\mathcal{E}_2$ and $\mathcal{E}_3$, 
and show why it is important to control them in practice. 
%significantly decrease the performance of DRL algorithm. Hence, it is important to deduct the variance $\mathrm{Var}(\mathcal{E}_1 + \mathcal{E}_2)$.

\subsection{ Parametrization Induced Error}
We first show the convergence of both the expectation and the variance of the distributional Bellman operator $\mathcal{T}^{\pi}$. Then, we take parametric representation and projection operator into consideration.
\begin{proposition}[\citealp{sobel1982,C51}]\label{totalconverge}
Suppose there are two value distributions $\nu_{1}, \nu_{2} \in \mathscr{P}(\mathbb{R})$, and random variables $Z_{i}^{k+1}\sim \mathcal{T}^{\pi} \nu_{i}, Z_{i}^{k}\sim \nu_{i} $. Then, we have
\begin{align*}
\left\|\mathbb{E}Z_{1}^{k+1}-\mathbb{E}Z_{2}^{k+1}\right\|_{\infty} \leq \gamma\left\|\mathbb{E}Z_{1}^{k}-\mathbb{E}Z_{2}^{k}\right\|_{\infty}, \text { and } \\
\left\|\mathrm{Var}Z_{1}^{k+1}-\mathrm{Var}Z_{2}^{k+1}\right\|_{\infty} \leq \gamma^{2}\left\|\mathrm{Var}Z_{1}^{k}-\mathrm{Var}Z_{2}^{k}\right\|_{\infty}.
\end{align*}
%where $\|\cdot\|_\infty$ is the maximal form of the Wasserstein metric.
\end{proposition}
Based on the fact that $\mathcal{T}^{\pi}$ is a $\gamma$-contraction in $\bar{d}_{p}$ metric \cite{C51}, where $\bar{d}_{p}$ is the maximal form of the Wasserstein metric, Proposition \ref{totalconverge} implies that $\mathcal{T}^{\pi}$ is a contraction for both the expectation and the variance. The two converge exponentially to their true values by iteratively applying the distributional Bellman operator.

However, in practice, employing parametric representation for the return distribution leaves a theory-practice gap, which makes neither the expectation nor the variance converge to the true values. To better understand the bias in the Q-function approximation caused by the parametric representation, we introduce the concept of {\it mean-preserving} \footnote{This property has been thoroughly discussed in previous work. Based on Section 5.11 of \citet{DRLbook}, we conclude this definition.} to describe the relationship between the expectations of the original distribution and the projected distribution:
%Before we show how the parametric representation and its associated operator results in the estimation bias of the Q-function,
\begin{definition}[\textbf{Mean-preserving}] Let $\Pi_{\mathscr{F}}:\mathscr{P}(\mathbb{R})\to \mathscr{F}$ be a projection operator that maps the space of probability distributions to the desired representation. Suppose there is a representation $\mathscr{F}\in \mathscr{P}(\mathbb{R}) $ and its associated projection operator $\Pi_{\mathscr{F}}$ are mean-preserving if for any distribution $\nu\in\mathscr{F}$, the expectation of $\Pi_{\mathscr{F}}\nu$ is the same as that of $\nu$.
\end{definition}

For CDRL, a discussion of the {\it mean-preserving} property is given by \citet{lyle2019} and \citet{ EDRL}. It can be shown that for any $\nu\in\mathscr{F}_{\mathcal{C}}$, where $\mathscr{F}_{\mathcal{C}}$ is a $N$-categorical representation, the projection $\Pi_{\mathcal{C}}$ preserves the distribution's expectation when its support is contained in the interval $[z_{1}, z_{N}]$. However, these practitioners usually employ a wide predefined interval for return which makes the projection operator typically overestimate the variance.

For QDRL, $\Pi_{\mathcal{W}_{1}}$ is not {\it mean-preserving} \cite{DRLbook}. Given any distribution $\nu\in\mathscr{F}_{\mathcal{W}_{1}}$, where $\mathscr{F}_{\mathcal{W}_{1}}$ is a $N$-quantile representation, there is no unique $N$-quantile distribution $\Pi_{\mathcal{W}_{1}}\nu$ in most cases, as the projection operator $\Pi_{\mathcal{W}_{1}}$ is not a non-expansion in 1-Wasserstein distance (See \cref{proof} for details). This means that the expectation, variance, and higher-order moments are not preserved.  To make this concrete, a simple MDP example is used to illustrate the bias in the learned quantile estimates. 

In \cref{toy_example} (a), rewards $R_1$ and $R_2$ are randomly sampled from $\mathrm{Unif}(0,1)$ and $\mathrm{Unif}(1/N, 1+1/N)$ at states $x_{1}$ and $x_{2}$ respectively, and no rewards are received at $x_{0}$. Clearly, the true return distribution at state $x_{0}$ is the mixture $\frac{\gamma}{2}( R_{1}+ R_{2})$, hence the $\frac{1}{2 N}$-th quantile is $\frac{\gamma}{N}$. When using the QDRL algorithm with $N$ quantile estimates, the approximated return distribution $\hat{\eta}(x_1,a)=\frac{1}{N} \sum_{i=1}^{N} \delta_{\frac{2 i-1}{2 N}}$ and $\hat{\eta}(x_2,a)=\frac{1}{N} \sum_{i=1}^{N} \delta_{\frac{2 i+1}{2 N}}$. In this case, the $\frac{1}{2 N}$-th quantile of the approximated return distribution at state $x_{0}$ is $\frac{3 \gamma}{2 N}$, whereas the true value is $\frac{\gamma}{N}$. Moreover, for each $i=1, \ldots, N$, the $\frac{2i-1}{2 N}$-th quantile estimate at state $x_{0}$ is not equal to the true value. 
%we illustrate using a simple MDP example, where the learned quantile estimates can be biased. 

%The illustrating example in \cref{fig1} (a) presents the biased quantile approximation caused by $\mathcal{E}_3$. To be more specific, we now derive a upper bound for $\mathcal{E}_3$ during the QDRL update.
\begin{figure}[ht]
%\vskip -0.1in
\centerline{\includegraphics[width=1\linewidth]{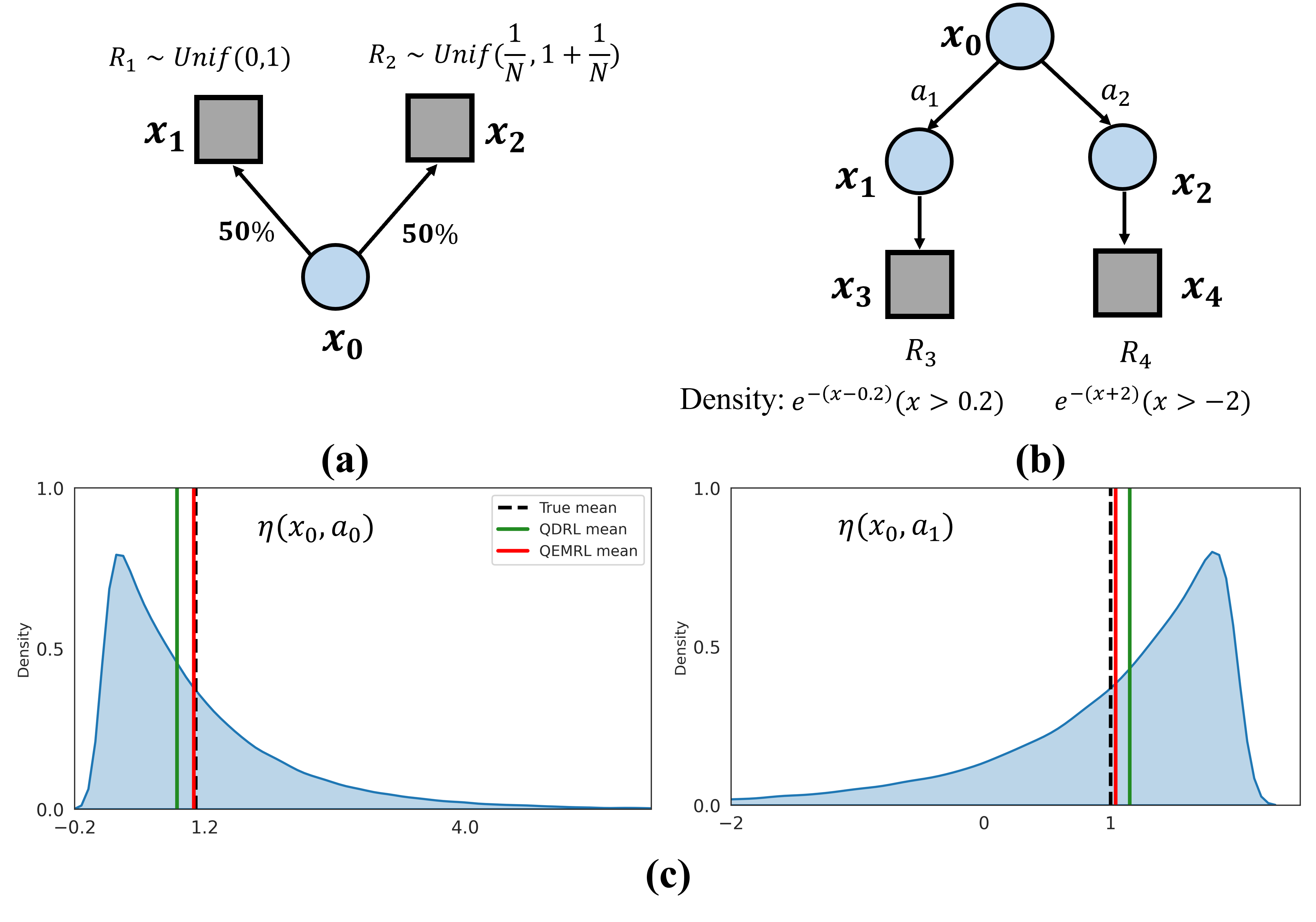}}
\vskip -0.05in
\caption{(a) Example MDP, with a single action, equal transition probability, an initial state $x_{0}$, and two terminal states $x_{1}, x_{2}$ where rewards are drawn from uniform. (b) 5-state MDP, with two actions at initial state $x_{0}$, deterministic transition, and stochastic rewards are exponential at terminal states $x_{3}, x_{4}$. (c) We show the true return distributions $\eta(x_{0}, a_{1})$ and $\eta(x_{0}, a_{2})$, and the expected returns estimated by QDRL and QEMRL.}
\label{toy_example}
\vskip -0.05in
\end{figure}

These biased quantile estimates illustrated in \cref{toy_example} (a) are caused by the use of quantile representation and its projection operator $\Pi_{\mathcal{W}_{1}}$. This undesirable property in turn affects the QDRL update, as the combined operator $\Pi_{\mathcal{W}_{1}}\mathcal{T}^{\pi}$ is in general not a non-expansion in $\bar{d}_{p}$, for $p\in[1, \infty)$ \cite{QRDQN}, which means that the learned quantile estimates may not converge to the true quantiles of the return distribution \footnote{ A recent study \cite{rowland2023analysis} proves that QDRL update may have multiple fixed points, indicating quantiles may not converge to the truth. Despite this, Proposition 2 \cite{QRDQN} concludes that the projected Bellman operator $\Pi_{\mathcal{W}_{1}}\mathcal{T}^{\pi}$ remains a contraction in $\bar{d}_{\infty}$. This implies that quantile convergence is guaranteed for all $p\in[1,\infty]$.}. The projection operator $\Pi_{\mathcal{W}_{1}}$ is not {\it mean-preserving}
%does not preserve mean 
which inevitably leads to bias in the expectation of return distribution when iteratively applying the projected Bellman operator $\Pi_{\mathcal{W}_{1}}\mathcal{T}^{\pi}$ during the training process, resulting in a deviation between the estimate and the true value of the Q-function in the end. We now derive an upper bound to quantify this deviation, i.e. $\mathcal{E}_3$.

%and the Thus, we want to do some further theoretical analysis of $\mathcal{E}_3$ and try to derive an upper bound for $\mathcal{E}_3$.
\begin{theorem}[\textbf{Parameterization induced error bound}]\label{bound}
 Let $\Pi_{\mathcal{W}_{1}}$ be a projection operator onto evenly spaced quantiles $\tau_{i}$'s where each 
%$\tau_{i}$ th quantile, 
$\tau_{i}=\frac{2 i-1}{2 N}$ for $i=1, \ldots, N$, and $\eta_{k}\in\mathscr{P}(\mathbb{R})$ be the return distribution of $k$-th iteration. Let random variables $Z_{\theta}^{k}\sim\Pi_{\mathcal{W}_{1}}\mathcal{T}^{\pi}\eta_{k}$ and $Z^{k}\sim\mathcal{T}^{\pi}\eta_{k}$. Assume that the distribution of the immediate reward is supported on $[- R_{max}, R_{max}]$, then we have
\begin{align*}
\lim _{k \rightarrow \infty} \left\|\mathcal{E}^k_3 \right\|_{\infty} = \lim _{k \rightarrow \infty}\left\|\mathbb{E}Z_{\theta}^{k}-\mathbb{E}Z^{k}\right\|_{\infty} \leq \frac{2 R_{max}}{N (1-\gamma)},
%\lim _{k \rightarrow \infty}\left\|\mathbb{E}[Z_{\theta}^{k}]^{2}-\mathbb{E}[Z^{k}]^{2}\right\|_{\infty} \leq \frac{R_{max}^{2}}{N (1-\gamma)}.
\end{align*}
where $\mathcal{E}^k_3$ is parametrization induced error at $k$-th iteration.
\end{theorem}

\cref{bound} implies that the convergence of expectation with projected Bellman operator $\Pi_{\mathcal{W}_{1}}\mathcal{T}^{\pi}$ cannot be guaranteed after quantile representation and its projection operator are applied \footnote{Note that this bound has a limitation, which only considers the one-step effect of applying the projection operator $\Pi_{W_1}$. Therefore, it becomes irrelevant with the iteration number $k$. However, Proposition 4.1 of \citet{rowland2023statistical} provides a more compelling bound considering the cumulative effect of iteratively applying $\Pi_{W_1}$.}. Note that the bound will tend to zero with $N \rightarrow \infty$, thus it is reasonable to use a relatively large representation size $N$ to reduct $\mathcal{E}_3$ in practice. 

\subsection{Approximation Error}

The other two types of errors $\mathcal{E}_1$ and $\mathcal{E}_2$, which determine the variance of the Q-function estimate, are accumulated during the training process by keeping encountering unseen state-action pairs. The target approximation error $\mathcal{E}_1$ affects action selections, while the Bellman operator approximation error $\mathcal{E}_2$ leads to the accumulated error of the Q-function estimate,
%cause errors in the Q-function to accumulate over time
which can be amplified by using the temporal difference updates \cite{Sutton1988LearningTP}. The accumulated errors of the Q-function estimate with high uncertainty can 
%lead to a significant bias making 
make some certain states to be incorrectly estimated, leading to suboptimal policies and potentially divergent behaviors.

As depicted in \cref{toy_example} (b), we utilize this toy example to illustrate how QDRL fails to learn an optimal policy due to a high variance of the approximation error. This 5-state MDP example is originally introduced in Figure 7 of \citet{EDRL}. In this case, $\eta(x_{0}, a_{1})$ and $\eta(x_{0}, a_{2})$ follow exponential distributions, and the expectations of them are 1.2 and 1, respectively. We consider a tabular setting, which uniquely represents the approximated return distribution at each state-action pair. \cref{toy_example} (c) demonstrates that in policy evaluation, QDRL inaccurately approximates the Q-function, as it underestimates the expectation of $\eta(x_{0}, a_{1})$ and overestimates the other. This is caused by the poor capture of tail events, which results in high uncertainty in the Q-function estimate. Due to the high variance, QDRL fails to learn the optimal policy and chooses a non-optimal action 
%This high variance estimate results in QDRL learning a suboptimal policy that chooses 
$a_{2}$ at the initial state $x_{0}$.
%, in a control task. 
On the contrary, our proposed algorithm, QEMRL, employs a statistically robust estimator of the Q-function to reduce its variance, relieves the underestimation and overestimation issues, and ultimately allows for more efficient policy learning.

Different from previous QDRL studies that focus on exploiting the distribution information to further improve the model performance, this work highlights the importance of controlling the variance of the approximation error to obtain a more accurate estimate of the Q-function. 
%accelerate the reduction of errors during training.
%Contrary to the popular belief that quantile-based DRL leads to a superior understanding of distributional knowledge and better performance, our findings highlight the importance of controlling the variance of approximation errors to accelerate the reduction of errors during training. 
More discussion about this is given in the following section. 
%as illustrated in \cref{illustration}. %In the next section, combining the phenomenon of quantile crossing and tail unrealization, we will elaborate on how noise on quantile approximation affects the Q-function approximation and introduce a novel technique to reduce the variance of $\Delta$. 

%\cref{fig1} (c) shows that, in policy evaluation, QDRL actually underestimates the expectation of $\eta(x_{0}, a_{1})$ and overestimates the other due to the poor tail estimation of exponential distribution.

%%%%%%%%%%%%%%%%%%%%%%%%%%%%%%%%%%%%%%%%%%%%%%%%%%%%%%%%%%%%%%%%%%%%%%%%%%%%%%%
 %                         SECTION  4
%%%%%%%%%%%%%%%%%%%%%%%%%%%%%%%%%%%%%%%%%%%%%%%%%%%%%%%%%%%%%%%%%%%%%%%%%%%%%%%

\section{Quantiled Expansion Mean}\label{sec4}

This section introduces a novel variance reduction technique to estimate the Q-function. In traditional statistics, estimators with lower variance are considered to be more efficient. In RL, variance reduction is also an effective technique for achieving fast convergence in both policy-based and value-based RL algorithms, especially for large-scale tasks \cite{greensmith2004variance, AVERAGED-DQN}. Motivated by these findings, we introduce QEM as an estimator that is more robust and has a lower variance than that of QDRL under the heteroskedasticity assumption. Furthermore, we demonstrate the potential benefits of QEM for the distribution approximation in DRL. %We begin by discussing the property of heteroskedasticity of quantiles.

\subsection{Heteroskedasticity of quantiles}

In the context of quantile-based DRL, Q-function is the integral of the quantiles. To approximate this, QDRL employs a simple empirical mean (EM) estimator $\frac{1}{N}\Sigma_{i}\hat{q}(\tau_{i})$, and it is natural to assume that the estimated quantile satisfies
\begin{align}\label{error model}
\hat{q}(\tau) = q(\tau) + \varepsilon(\tau),
\end{align} 
where $\varepsilon(\tau)$ is a zero-mean error. In this case, considering the crossing issue and the biased tail estimates, we assume that the variance of $\varepsilon(\tau)$ is non-constant and depends on $\tau$, which is usually called heteroskedasticity in statistics.
%should pay attention to the non-constant variance $\varepsilon(\tau)$ depending on $\tau$, which is called heteroskedasticity. 
%and is a hot topic for statisticians.
%To approximate this
%Since the whole distribution is parametrized by finite $N$ quantiles,
%We assume that the estimated quantile can be represented by the equation:
%\begin{align}\label{error model}
%\hat{q}(\tau) = q(\tau) + \varepsilon(\tau),
%\end{align} 
%where $q(\tau)$ is the fixed point and the zero-mean error $\varepsilon(\tau)$ has variance relevant with $\tau$. This property of heteroskedasticity, resulting from the fact that quantiles are only dependent on frequency, can explain the phenomenon of crossing and tail unrealization.

%\begin{lemma}\label{expectation} (expectation by quantiles). Let $Z\sim\nu\in\mathscr{P}(\mathbb{R})$ be a random variable with CDF $F_{\nu}$ and quantile function $F_{\nu}^{-1}$. Then,
%$$
%\mathbb{E}[Z]=\int_{0}^{1} F_{\nu}^{-1}(\tau) d \tau .
%$$
%\end{lemma}

For a direct understanding, we conduct a simple simulation using a Chain MDP to illustrate how QDRL can fail to fit the quantile function. As shown in \cref{fig2}(b), 
%it is clear that 
QDRL fits well in the peak area
%areas of the peak 
but struggles at the bottom and the tail. Moreover, the non-monotonicity of the quantile estimates in the poorly fitted areas is more severe than the others. 
%with crossing occurring in poorly fitted areas. 
As the deviations of the quantile estimates from the truths is significantly larger in the low probability region and the tail, we can make the heteroskedasticity assumption in this case. 
%the heteroskedasticity assumption is satisfied in this case.
%as the approximated quantiles are randomly scattered around the truths, exhibiting large deviations near the low probability and tail regions while remaining relatively close in others. 
This phenomenon can be explained since samples near the bottom and the tail are less likely to be drawn. In real-world situations, multimodal distributions are commonly encountered and the heteroskedasticity problem may result in imprecise distribution approximations and consequently poor Q-function approximations. In the next part, we will discuss how to enhance the stability of the Q-function estimate.

\subsection{Cornish-Fisher Expansion}\label{QCM}

It is well-known that quantile can be expressed by the Cornish-Fisher Expansion (CFE, \citealp{Fisher1938}):
% %In financial management, 
% The Quantile
% %also known as value-at-risk, 
% %(VaR), 
% can be expressed by the Cornish-Fisher Expansion (CFE, \citealp{Fisher1938})
%as an asymptotic series
\begin{align}\label{CFE}
&q(\tau)=\mu+\sigma x'_{\tau},\\
&x'_{\tau} =z_{\tau}+(z^{2}_{\tau}-1)\frac{s}{6} +(z^{3}_{\tau}-3z_{\tau}) \frac{k}{24} +\cdots ,\nonumber
\end{align}
where $z_{\tau}$ is the $\tau$-th quantile of the standard normal distribution, $\mu$ is the mean, $\sigma$ is the standard deviation, $s$ and $k$ are the skewness and kurtosis of the interested distribution, and the remaining terms in the ellipsis are higher-order moments (See \cref{intro_CFE} for more details). The CFE 
theoretically determines the distribution with known moments and is widely used in financial studies. Recently, \citet{qcm} employ CFE to estimate higher-order moments of financial time series data, which are not directly observable. Our method utilizes a truncated version of CFE framework and employs a linear regression model to construct efficient estimators for distribution moments based on known quantiles. Consequently, we apply this approach within the context of quantile-based DRL. 

% Our method considers a truncated version of CFE, and uses linear regression to construct efficient estimators for distribution moments by known quantiles within the context of quantile-based DRL.

% the distribution moments with known quantiles within the context of quantile-based DRL.

% in quantile-based DRL.

To be more specific, 
%for the $\tau$-th quantile $q(\tau)$, 
%of the return distribution at a certain state-action pair, 
we plug in the estimate $\hat{q}({\tau})$ of the the $\tau$-th quantile 
%$q(\tau)$ 
to \cref{CFE} and expand it by the first order:
\begin{align}\label{eqcf}
\hat{q}(\tau) = & m_1 + \omega_1(\tau) + \varepsilon(\tau),
\end{align}
%\bfblue{I suggest to modify $\sqrt{m_2} \omega_1(\tau)$ into $\omega_1(\tau)$, since the total term $\sqrt{m_2} \omega_1(\tau)$ is the remained term of first-order expansion.}

where $m_1$ is the mean (say, $1$-th moment) of the return distribution, i.e., the Q-function, %$m_2$ is the variance, 
and $\omega_1(\tau)$ is the remaining term associated with the higher-order ($>1$-th) moments. If $\omega_1(\tau)$ is negligible, $m_1$ can be estimated by averaging the $N$ quantile estimates in QDRL.
%this equation reduces to \cref{error model} and we can also obtain the EM estimator for the Q-function.

%We then examine the scenario where 
%When the estimated quantile is expanded to the second order. Particularly, we have the following representation:

When the estimated quantile is expanded to the second order, we particularly have the following representation:
\begin{align}
\hat{q}(\tau) = & m_{1} + z_{\tau} \sqrt{m_{2}} + \sqrt{m_{2}} \omega_2(\tau) + \varepsilon(\tau),
\end{align}
where $\omega_2(\tau)$ is the remaining term associated with the higher-order ($>2$-th) moments. Assume that $\omega_2(\tau)$ is negligible, we can derive a regression model by plugging in the $N$ quantile estimates, such that
\begin{align}
\label{reg}
\left(\begin{array}{c}
\hat{q}(\tau_{1}) \\
\hat{q}(\tau_{2}) \\
\vdots \\
\hat{q}(\tau_{N})
\end{array}\right) = \left(\begin{array}{cccc}
1 & z_{\tau_1} \\
1 & z_{\tau_2} \\
\vdots & \vdots \\
1 & z_{\tau_N}  
\end{array}\right)
\left(\begin{array}{c}
m_{1} \\
\sqrt{m_{2}}
\end{array}\right) 
+\left(\begin{array}{c}
\varepsilon(\tau_{1}) \\
\varepsilon(\tau_{2})\\
\vdots \\
\varepsilon(\tau_{N}) 
\end{array}\right).
\end{align}
The higher-order expansions can be conducted in the same manner. Note that the remaining term is omitted for constructing a regression model, and a more in-depth analysis of the remaining term is available in \cref{remaining term}.

For notation simplicity, we rewrite (\ref{reg}) in a matrix form,
%it as a linear model
\begin{align}\label{linear model}
\boldsymbol{\hat{Q}}=\mathbf{X}_2 \boldsymbol{M}_2+\mathcal{E},
\end{align}
%where $\boldsymbol{\hat{Q}} \in \mathbb{R}^N$ is the vector of estimated quantiles, $\mathbf{X}_2 \in \mathbb{R}^{N \times 2}$ and $\boldsymbol{M}_2 \in \mathbb{R}^2$ are the design matrix and the moments. $\mathcal{E}$ is the vector of error terms. 
where $\boldsymbol{\hat{Q}} \in \mathbb{R}^N$ is the vector of estimated quantiles, $\mathbf{X}_2 \in \mathbb{R}^{N \times 2}$ and $\boldsymbol{M}_2 \in \mathbb{R}^2$ are the design matrix and the moments respectively, and $\mathcal{E}$ is the vector of error terms. 

%For $\tau$-th quantile $q(\tau)$ of return distribution at certain state-action pair, according to \cref{CFE} we have
%\begin{align}\label{eqcf}
%(\tau) = &m_{1} + z_{\tau} \sqrt{m_{2}}+\left(z_{\tau}^{2}-1\right) \frac{\sqrt{m_{2}} m_{3}}{6} \\ \nonumber
%&+\left(z_{\tau}^{3}-3 z_{\tau}\right) \frac{\sqrt{m_{2}}m_{4}}{24}+\sqrt{m_{2}} \omega(\tau),
%\end{align}
%where $m_{i}$ denotes $i$-th moment (e.g, $1$-th moment is mean) for the corresponding return distribution of state-action pair, and $\omega(\tau)$ denotes the remained term on the higher order moments. For this expansion, we can obtain the estimate $\hat{q}(\tau)$ through quantile-based DRL, which allows us to estimate the moments based on $\hat{q}(\tau)$. Consider \cref{eqcf} at multiple quantile levels simultaneously, 

\begin{figure}[!h]
%\vskip -0.1in
\centerline{\includegraphics[width=0.88\linewidth]{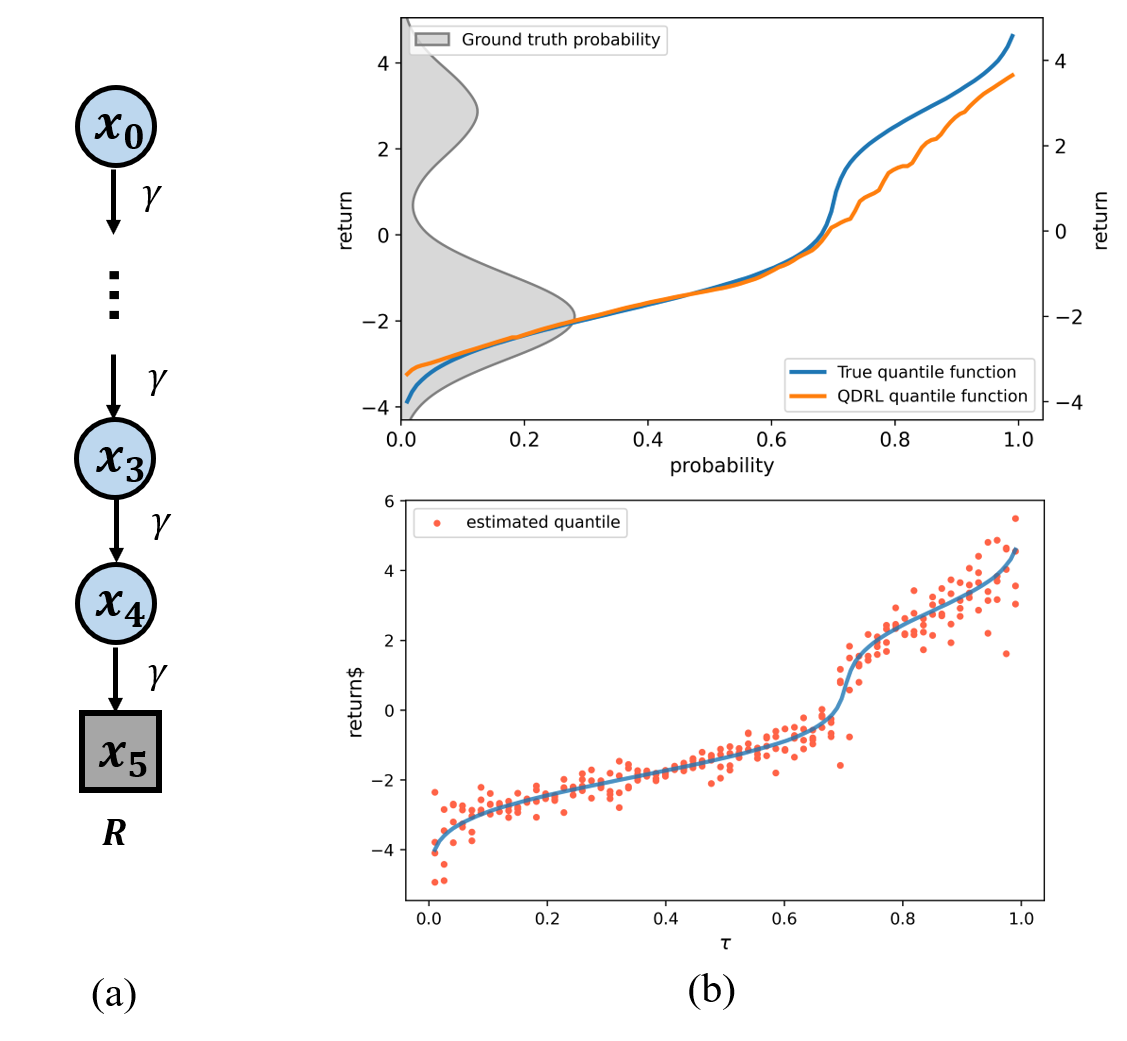}}
\vskip -0.1in
\caption{(a) Chain MDP, with six states, one action, $\gamma=0.99$ and gaussian mixture reward distribution
at terminal state $x_{5}$. (b) True quantile function (top) and QDRL quantile function at state $x_{0}$ after 10K steps iterate. Scatter diagram (bottom) of approximated quantile from training process.} %under different settings(initialization, seed, and training steps)
\label{fig2}
\vskip -0.1in
\end{figure}

%In statistics, linear regression model is a commonly used tool for predictive analysis assumes homoscedasticity, which means that the variances of the errors are invariant across different samples.  
%By assuming the homoscedasticity, 
For this bivariate regression model (\ref{linear model}), the traditional ordinary least squares method (OLS) can be used to estimate $\boldsymbol{M}_2 = (m_{1},\sqrt{m_{2}})'$ when the variances of the errors are invariant across different quantile locations, also known as the homoscedasticity assumption. The estimator $\hat{m}_1$ is denoted as Quantiled Expansion Mean (QEM) in this work. However, since the homoscedasticity assumption required by OLS is always violated in real cases, we may consider using the weighted ordinary least squares method (WLS) instead.  Under the normality assumption, the following results tell that the WLS estimator $\hat{m}_1$ has a lower variance than the direct empirical mean. 

%As the variances of the errors are invariant across different quantile locations, i.e. the homoscedasticity, (\ref{linear model}) is a bivariate regression model. 
%For this bivariate regression model, 
%The traditional ordinary least squares method (OLS) can be used to estimate $\boldsymbol{M}_2 = (m_{1},\sqrt{m_{2}})'$
%We use the ordinary least squares method (OLS) to estimate $\boldsymbol{M}_2 = (m_{1},\sqrt{m_{2}})'$ and denote
%we refer to \
%and the estimator $\hat{m}_1$ is denoted as Quantile Expansion Mean (QEM) in this work. However, since the homoscedasticity assumption required by OLS is always violated in real cases, we may consider using the 
%To address this, it would be reasonable to use the 
%weighted ordinary least squares method (WLS) instead.  Under the normality assumption, the following results tell that the WLS estimator $\hat{m}_1$ has a lower variance than the direct empirical mean. 

%However, this linear model violates the assumption of homoscedasticity required in linear regression. To eliminate the unequal variance, it would be reasonable to use the weighted ordinary least squares method (WLS).
\begin{lemma}
\label{lem:heteroskedastic}
Consider the linear regression model $\boldsymbol{\hat{Q}}=\mathbf{X}_2 \boldsymbol{M}_2+\mathcal{E}$,  $\mathcal{E}$ is distributed on $\mathcal{N}(\mathbf{0},\sigma^{2} V)$, where $V=diag(v_{1},v_{2},\cdots,v_{N}),v_{i}\geq 1, i=1,\cdots,N$, and we set noise variance  $\sigma^2 = 1$ without loss of generality. 
The WLS estimator is
\begin{align}\label{WLS}
\widehat{\boldsymbol{M}}_2 = (\mathbf{X}^{\top}_2 V^{-1} \mathbf{X}_2)^{-1}\mathbf{X}^{\top}_2 V^{-1} \boldsymbol{\hat{Q}},
\end{align}
and the QEM estimator $\hat{m}_1$ is the first component of $\widehat{\boldsymbol{M}}_2$.
%and the distribution of mean estimator takes the form, $$\hat{m}_{1}\sim \mathcal{N}\left(m_{1},\frac{1}{\sum_{i}v_{i}}+\frac{(\frac{\sum_{i}v_{i}z_{\tau_i}}{\sum_{i}v_{i}})^2}{\sum_{i}v_{i}z_{\tau_i}^2-\frac{(\sum_{i}v_{i}z_{\tau_i})^2}{\sum_{i}v_{i}}}\right).$$  When $V$ equals to the identity matrix $I$, $$\hat{m}_{1}\sim \mathcal{N}\left(m_{1},\frac{1}{N}+\frac{\bar{z}^2}{\sum_{i}(z_{\tau_i}-\bar{z})^2}\right).$$
\end{lemma}
\textbf{Remark:} Note that it is impossible to determine the weight matrix $V$ for each state-action pair in practice. Hence, we focus on capturing the relatively high variance in the tail,
%of the distribution, 
specifically in the range of $\tau \in (0,0.1]\cup[0.9,1)$. To achieve this, we use a constant $v_{i}$, which is set to a value greater than 1 in the tail and equal to 1 in the rest. $v_{i}$ is treated as a hyperparameter to be tuned in practice
%tuned for benchmarks in experiments 
(See \cref{additional experiments}). 

With \cref{lem:heteroskedastic}, the reduction of variance can be guaranteed by the following \cref{propostion1}. Throughout the training process, heteroskedasticity is inevitable, and thus the QEM estimator always exhibits a lower variance than the standard EM estimator $\hat{m}^*_1 = \frac{1}{N} \sum_{i = 1}^N \hat{q}(\tau_i)$.
\begin{proposition}\label{propostion1}
Suppose the noise $\varepsilon_{i}$ independently follows $\mathcal{N}(0, v_{i})$ where $v_{i}\geq 1$ for $i=1,\cdots,N$, then, %$\mathrm{Var}(\hat{m}^{*}_{1})=\frac{\sum_i v_i}{N^2}$. 

(i) In the homoskedastic case where $v_{i}=1$ for $i = 1, \dots N$, the empirical mean estimator $\hat{m}_{1}^{*}$ has a lower variance, $\mathrm{Var}(\hat{m}_{1}^{*})<\mathrm{Var}(\hat{m}_{1})$ ; 

(ii) In the heteroskedastic case where $v_{i}$'s are not eaqul, the QEM estimator $\hat{m}_{1}$ achieves a lower variance, i.e. $\mathrm{Var}(\hat{m}_{1})<\mathrm{Var}(\hat{m}_{1}^{*})$, if and only if $\bar{v}^2-1-1/(\frac{(\sum_{i}v_{i}\sum_{i}v_{i}z_{\tau_i}^2)}{(\sum_{i}v_{i}z_{\tau_i})^2}-1)>0$, where $\bar{v}=\frac{1}{N}\sum_{i}v_{i}$. This inequality holds when $z_{\tau_i} = -z_{\tau_{N-i}}$, which can be guaranteed in QDRL. 

\end{proposition}

We also try to explore the potential benefits of the variance reduction technique QEM in improving the approximation accuracy. The Q-function estimate with higher variance can lead to 
%High variance Q-function estimates can cause 
noisy policy gradients in policy-based algorithms \cite{TD3} and prevent selection optimal actions in value-based algorithms \cite{AVERAGED-DQN}. These issues can slow down the learning process and negatively impact the algorithm performance. By the following theorem, we are able to show that QEM can reduce the variance and thus improve the approximation performance. 
%In the following theorem, we give a bound on the concentration of the empirical CDF to the true CDF.
\begin{theorem}\label{theorem:concentration}
Consider the policy $\hat{\pi}$ that is learned policy, and
denote the optimal policy to be $\pi_{opt}$, $\alpha=\max_{x'} D_{TV}(\hat{\pi}\left(\cdot\mid x^{\prime})\| \pi_{opt}(\cdot \mid x^{\prime})\right)$, and $n(x,a)=|\mathcal{D}|$. For all $\delta \in \mathbb{R}$, with probability at least $1-\delta$, for any $\eta(x,a) \in \mathscr{P}(\mathbb{R})$, and all $(x, a) \in \mathcal{D}$,
\begin{small}
\begin{align*}
\left\|F_{\hat{\mathcal{T}}^{\hat{\pi}} \eta(x, a)}-F_{\mathcal{T}^{\pi_{opt}} \eta(x, a)}\right\|_{\infty}\leq 2\alpha + \sqrt{\frac{1+4|\mathcal{X}|}{n(x, a)} \log \frac{4|\mathcal{X}||\mathcal{A}|}{\delta}}.
\end{align*}
\end{small}
\end{theorem}
\cref{theorem:concentration} indicates that 
%accelerating of policy learning potentially improves distribution approximation. In other words, the 
a lower concentration bound can be obtained with a smaller $\alpha$ value. The decrease in $\alpha$ can be attributed to the benefits of QEM. Specifically, QEM helps to decrease the perturbations on the Q-function and reduce the variance of the policy gradients, which allows for faster convergence of the policy training and a more accurate distribution approximation. To conclude, QEM relieves the error accumulation within the Q-function update, improves the estimation accuracy, reduces the risk of underestimation and overestimation, and thus ultimately enhances the stability of the whole training process.

%a reasonable choice of expansion order is $4$ (i.e., explain the quantile values with mean, variance, skewness and kurtosis), since including more terms does not result in a significant increase in R-squared value.

\begin{algorithm}[tb]
   \caption{QEMRL update algorithm}
   \label{alg:QEMRL}
\begin{algorithmic}[1]
   \STATE {\bfseries Require:} Quantile estimates $\hat{q}_{i}(x, a)$ for each $(x, a)$ %current distribution $\eta(x,a)=\frac{1}{N}\sum_{i=1}^{N}\delta_{\hat{q}_{i}(x, a)}$ %$(x, a) \in$ $\mathcal{X} \times \mathcal{A}$ and $i=1, \ldots, N$
\STATE  Collect sample $(x, a, r, x^{\prime})$
 \STATE {\color{gray} \# Compute distributional Bellman target}
  \STATE Compute $Q(x^{\prime}, a)$ using \cref{WLS}
  \IF{policy evaluation}
  \STATE $a^{*} \sim \pi(\cdot|x^{\prime})$
  \ELSIF{Q-Learning}
  \STATE $a^{*} \leftarrow \arg\max_{a}Q(x^{\prime}, a)$
 \ENDIF
%\STATE {\color{gray} \# $ \hat{\eta}(x,a)\leftarrow \Pi_{\mathcal{W}}\left(f_{r, \gamma}\right)_{\#} \eta(x^{\prime},a^{*})$}
\STATE Scale samples $\hat{q}^{*}_{i}(x^{\prime},a^*) \leftarrow r+\gamma \hat{q}_{i}(x^{\prime},a^*)$, $\forall i$.
   
\STATE {\color{gray} \# Compute quantile loss}
 
   \STATE Update estimated quantiles $\hat{q}_{i}(x, a)$ by computing the gradients for each $i=1, \ldots, N$, $\nabla_{\hat{q}_{i}(x, a)}\sum_{i=1}^{N}\mathcal{L}_{QR}(\hat{q}_{i}(x, a);\frac{1}{N} \sum_{j=1}^{N} \delta_{\hat{q}^{*}_{j}(x^{\prime},a^*)},\tau_{i}).$ 
\end{algorithmic}
\end{algorithm}

%controls the approximation errors that accumulated within the Q-function update and highly reduces the risk of underestimation and overestimation.
%making it less susceptible to underestimation or overestimation. 
%To conclude, QEM relieves the error accumulation during the policy training, improves the estimation accuracy and ultimately enhances the stability of the whole training process. 

% to be less susceptible to underestimation or overestimation
% which can lead to biased Q-function and suboptimal policy updates

%%%%%%%%%%%%%%%%%%%%%%%%%%%%%%%%%%%%%%%%%%%%%%%%%%%%%%%%%%%%%%%%%%%%%%%%%%%%%%%
 %                         SECTION  5
%%%%%%%%%%%%%%%%%%%%%%%%%%%%%%%%%%%%%%%%%%%%%%%%%%%%%%%%%%%%%%%%%%%%%%%%%%%%%%%

\section{Experimental Results}\label{sec5}
In this section, we do some empirical studies to demonstrate the advantage of our QEMRL method. First, a simple tabular experiment is conducted to validate some of the theoretical results presented in \cref{sec3,sec4}. Then we apply the proposed QEMRL update strategy in \cref{alg:QEMRL} to both the DQN-style and SAC-style DRL algorithms, which are evaluated on the Atari and MuJoCo environments. The detailed architectures of these methods and the hyperparameter selections can be found in \cref{implement}, and the additional experimental results are included in \cref{additional experiments}.

In this work, we implement QEM using a $4$-th order expansion that includes mean, variance, skewness, and kurtosis in this work. 
%The implementation of QEM in this work uses a $4$-th order expansion that includes mean, variance, skewness, and kurtosis.  
The effects of a higher-order expansion on model estimation are discussed in \cref{model selection}. Intuitively, including more terms in the expansion improves the estimation accuracy of quantiles, but the overfitting risk and the computational cost are also increased. Hence, there is a trade-off between explainability and learning efficiency. We evaluate different expansion orders using the $R^2$ statistic, which measures the goodness of model fitting.
%onto the data.
The simulation results (\cref{regression R2}) show that a $4$-th order expansion seems to be the optimal choice while a higher-order ($>4$-th) expansion does not show a significant increase in $R^2$.

%We first present results with a tabular version of QCMRL to empirically validate the proposed method and expand upon the theoretical results presented in \cref{sec3,sec4}. To demonstrate the effectiveness of QCMRL at scale, we then combine the QCMRL update in \cref{alg:QCMRL} to the DQN-style and SAC-style architecture and evaluate it on the Atari and MuJoCo games. The specific details of the architectures and hyperparameters used can be found in the \cref{implement}.

%We give more details of the architectures and hyperparameters used in the
%\vskip -0.2in
\subsection{A Tabular Example}

FrozenLake \cite{openai} is a classic benchmark problem for Q-learning control with high stochasticity and sparse rewards, in which an agent controls the movement of a character in an $n\times n$ grid world.  As shown in \cref{Flake-example} with a FrozenLake-$4\times 4$ task, "S" is the starting point, "H" is the hole that terminates the game, "G" is the goal state with a reward of 1. All the blue grids stand for the frozen surface where the agent can slide to adjacent grids based on some underlying unknown probabilities when taking a certain movement direction. The reward received by the agent is always zero unless the goal state is reached.

We first approximate the return distribution under the optimal policy $\pi^*$, which can be realized using the value iteration approach. To be specific, we start from the "S" state and perform 1K Monte-Carlo (MC) rollouts. An empirical distribution can be obtained by summarizing all these recording trajectories. 
%learned by value iteration, at the start state by performing 1K Monte-Carlo (MC) rollouts and recording the observed returns as an empirical distribution. 
With the approximation of the distribution, we can draw a curve of 
%The empirical distribution is then transformed to achieve a curve of 
quantile estimates shown in Figure 5. Both QEMRL and QDRL were run for 150K training steps and the $\epsilon$-greedy exploration strategy is applied in the first 1K steps. For both methods, we set the total number of quantiles to be $N = 128$.
%using an $\epsilon$-greedy exploration strategy during the first 1K steps, with quantile numbers $N = 128$.

%1-Wasserstein metric
\begin{figure}[!h]
\vskip -0.1in
\centering
\includegraphics[width = 0.86\linewidth]{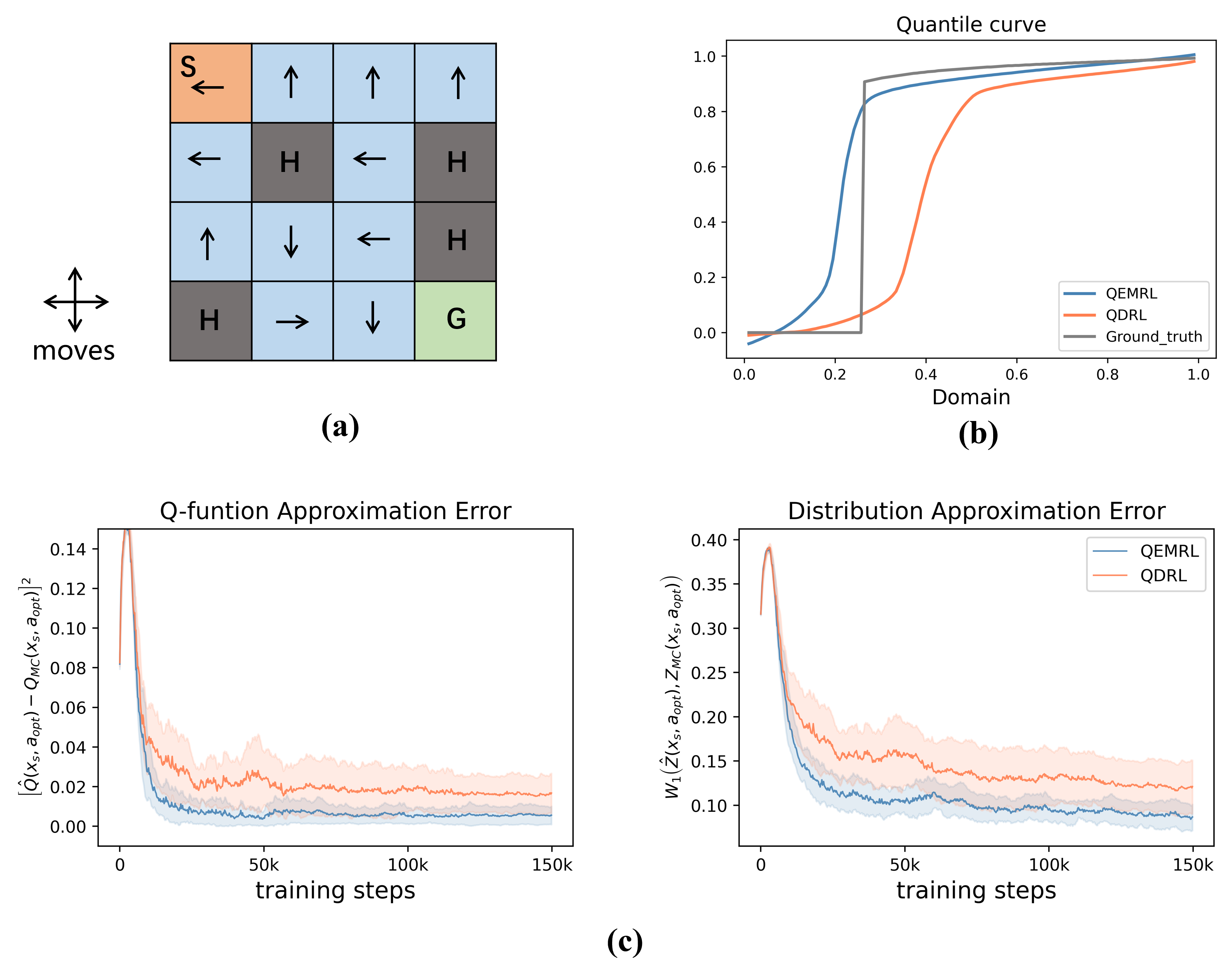}
\vskip -0.1in
\caption{(a) The optimal direction of movement at each grid.
%FrozenLake, with the movements of optimal policy shown at grids. 
(b) Quantile estimates by MC, QDRL, and QEMRL at the start state. (c) Approximation errors of Q-function estimate and distribution approximation error of QEMRL and QDRL (results are averaged over 10 random seeds). }
\label{Flake-example}
\vskip -0.1in
\end{figure}

Although both QEMRL and QDRL can eventually find the optimal movement at the start state, their approximations of the return distribution are quite different. \cref{Flake-example} (b) visualizes the approximation errors of the Q-function and the distribution for QEMRL and QDRL with respect to the number of training steps. The Q-function estimates of QEMRL converge correctly in average, whereas the estimates of QDRL do not converge exactly to the truth. A similar pattern can also be found when it comes to the distribution approximation error. Besides, the reduction of variance by using QEM can be verified by the fact that the curves of QEMRL are more stable and decline faster. In \cref{Flake-example} (c), we show that the distribution at the start state estimated by QEMRL is eventually closer to the ground truth.

\subsection{Evaluation on MuJoCo and Atari 2600}
We do some experiments using the MuJoCo benchmark to further verify the analysis results in \cref{sec4}. Our implementation is based on the Distributional Soft Actor-Critic (DSAC, \citealp{dsac}) algorithm, which is a distributional version of SAC. \cref{mujoco-qrdqn} demonstrate that both DSAC and QEM-DSAC significantly outperform the baseline SAC. Among the two, QEM-DSAC performs better than DSAC and the learning curves are more stable, which demonstrates that QEM-DSAC can achieve a higher sample efficiency.
%Furthermore, QCM-DSAC achieves higher sample efficiency, more stable learning curves, and improved performance compared to DSAC. 

\begin{figure}[!h]
%\vskip -0.1in
\centering
\includegraphics[width = 0.99\linewidth]{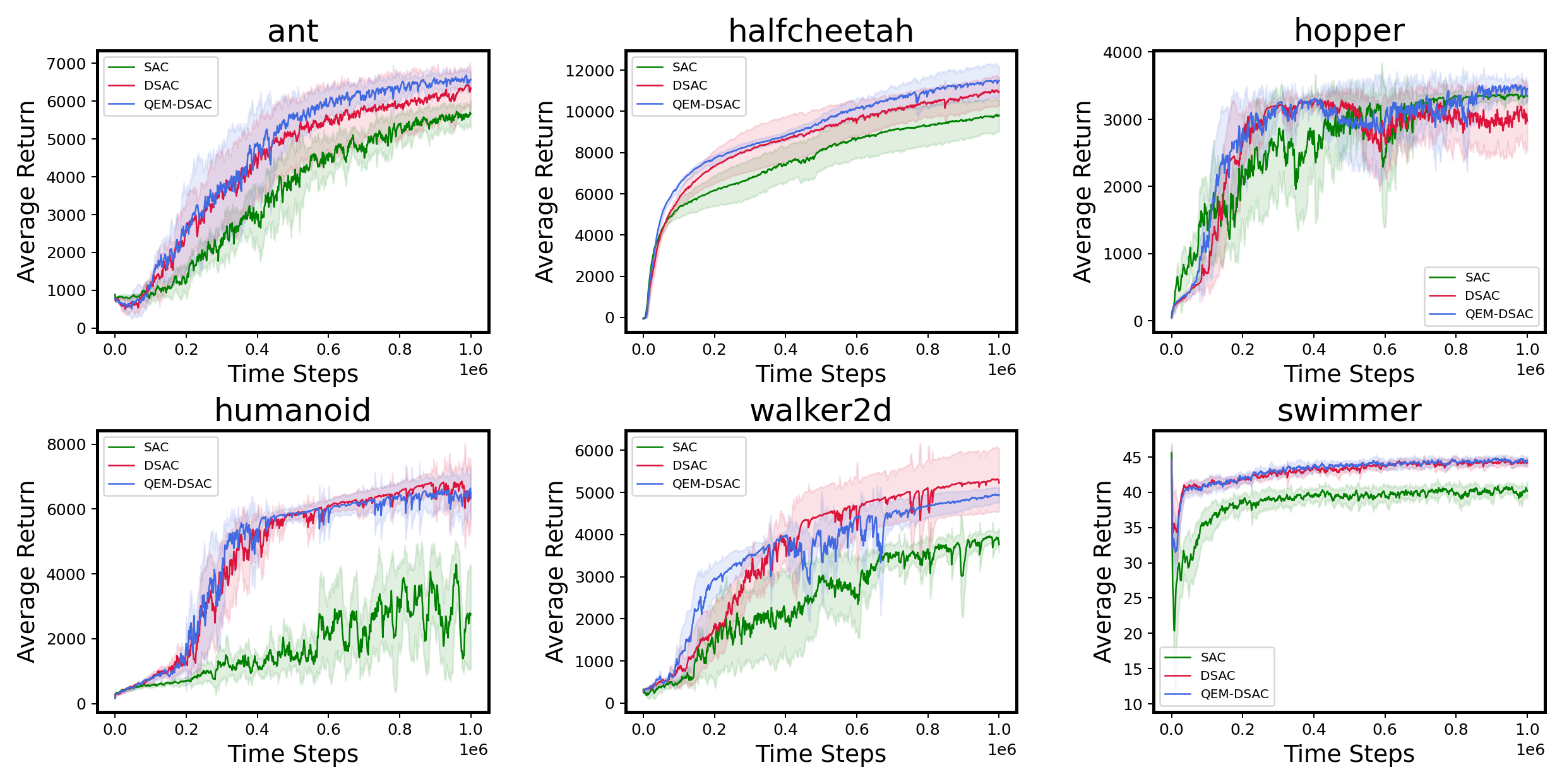}
\vskip -0.1in
\caption{ Learning curves of SAC, DSAC, and QEM-DSAC across six MuJoCo games. Each curve is averaged over 5 random seeds and shaded by their confidence intervals.}
\label{mujoco-qrdqn}
\vskip -0.05in
\end{figure}

%Given the superior properties of QCM derived from \cref{linear model},
%by the linear model \cref{linear model}, 

We also do some comparison between QEM and the baseline method QR-DQN on the Atari 2600 platform. 
%performed our method on the discrete control environment Atari 2600 games. 
%We choose QR-DQN as the baseline and compare it with QCM-DQN. 
\cref{atari} plots the final results of these two algorithms in six Atari games.
%In \cref{atari}, we show a comparison of these two algorithms on six games. 
At the early training stage, QEM-DQN exhibits significant gain in sampling efficiency, 
%During the early stages of training, QCM-DQN exhibits significant sample efficiency, 
resulting in faster convergence and better performance. 

\textbf{Extension to IQN.}
Some great efforts have been made by the community of DRL to more precisely parameterize the entire distribution with a limited number of quantile locations. 
%In the existing literature, great efforts have been taken to more precisely parameterize the entire distribution with a limited number of quantiles. 
One notable example is the introduction of Implicit Quantile Networks (IQN, \citealp{IQN}), which tries to recover the continuous map of the entire quantile curve by sampling a different set of quantile values from a uniform distribution $\mathrm{Unif}(0,1)$ each time.
%are proposed for more precisely parameterizing distribution to learn a continuous map from any sampled quantile fraction from $Unif(0,1)$ to its estimate on the quantile curve. 

Our method can also be applied to IQN as it uses the EM approach to estimate the Q-function. It is noted that the design matrix $\mathbf{X}$ must be updated after re-sampling all the quantile fractions at each training step. Moreover, one important sufficient condition $z_{\tau_i} = -z_{\tau_{N-i}}$ which ensures the reduction of variance does not hold in the IQN case as $\tau$'s are sampled from a uniform distribution. However, according to the simulation results in \cref{table4}, the variance reduction still remains valid in practice. 
%It is quite complex to prove this, however, simulation results (\cref{table4}) show that variance reduction remains valid in practice. 
In this case, all the baseline methods are modified to the IQN version. As \cref{mujoco-iqn} and \cref{iqcm} demonstrate, QEM can achieve some performance gain in most 
%DSAC is modified to the IQN version which maps the sampling quantile fractions to the full quantile function. As \cref{mujoco-iqn} shows, QCN-DSAC can achieve some performance gain than DSAC in most 
scenarios and the convergence speeds can be slightly increased. 
%has an IQN version variant that adopts the sampling quantile fractions mapping to the full quantile function, and the IQN version experiments are presented in \cref{mujoco-iqn}, which show that while the scores are slightly improved, the convergence speeds are significantly  improved significantly.

%Instead of using fixed quantile fractions, DSAC has an IQN version variant that adopts the sampling quantile fractions mapping to the full quantile function. 

\begin{figure}[!h]
%\vskip -0.1in
\includegraphics[width = 0.99\linewidth]{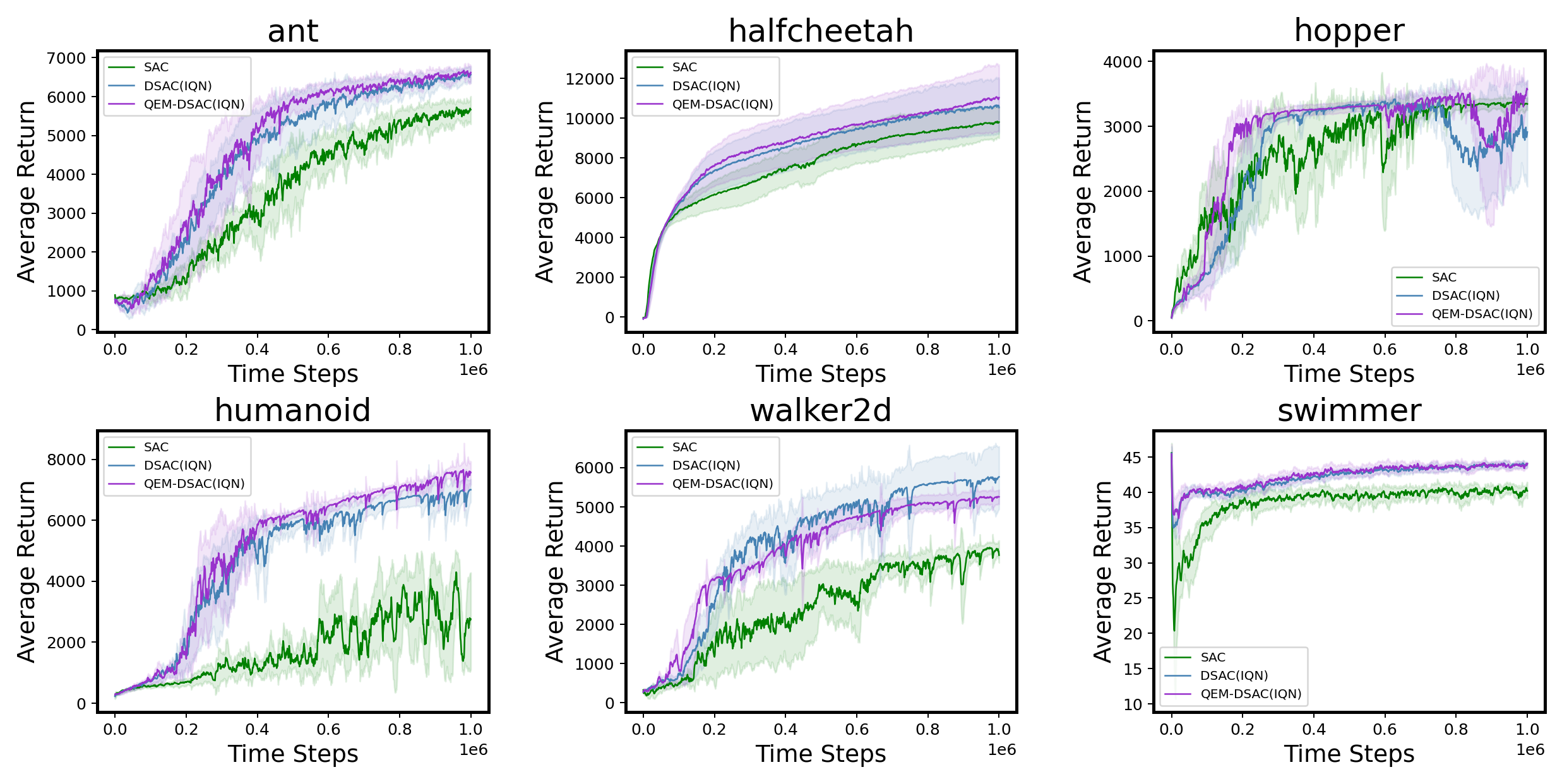}
\vskip -0.1in
\caption{Learning curves of SAC, DSAC (IQN), and QEM-DSAC (IQN) across six MuJoCo games. Each curve is averaged over 5 random seeds and shaded by their confidence intervals.}
\label{mujoco-iqn}
\vskip -0.1in
\end{figure}

\begin{figure}[!h]
\centering
\includegraphics[width = 0.99\linewidth]{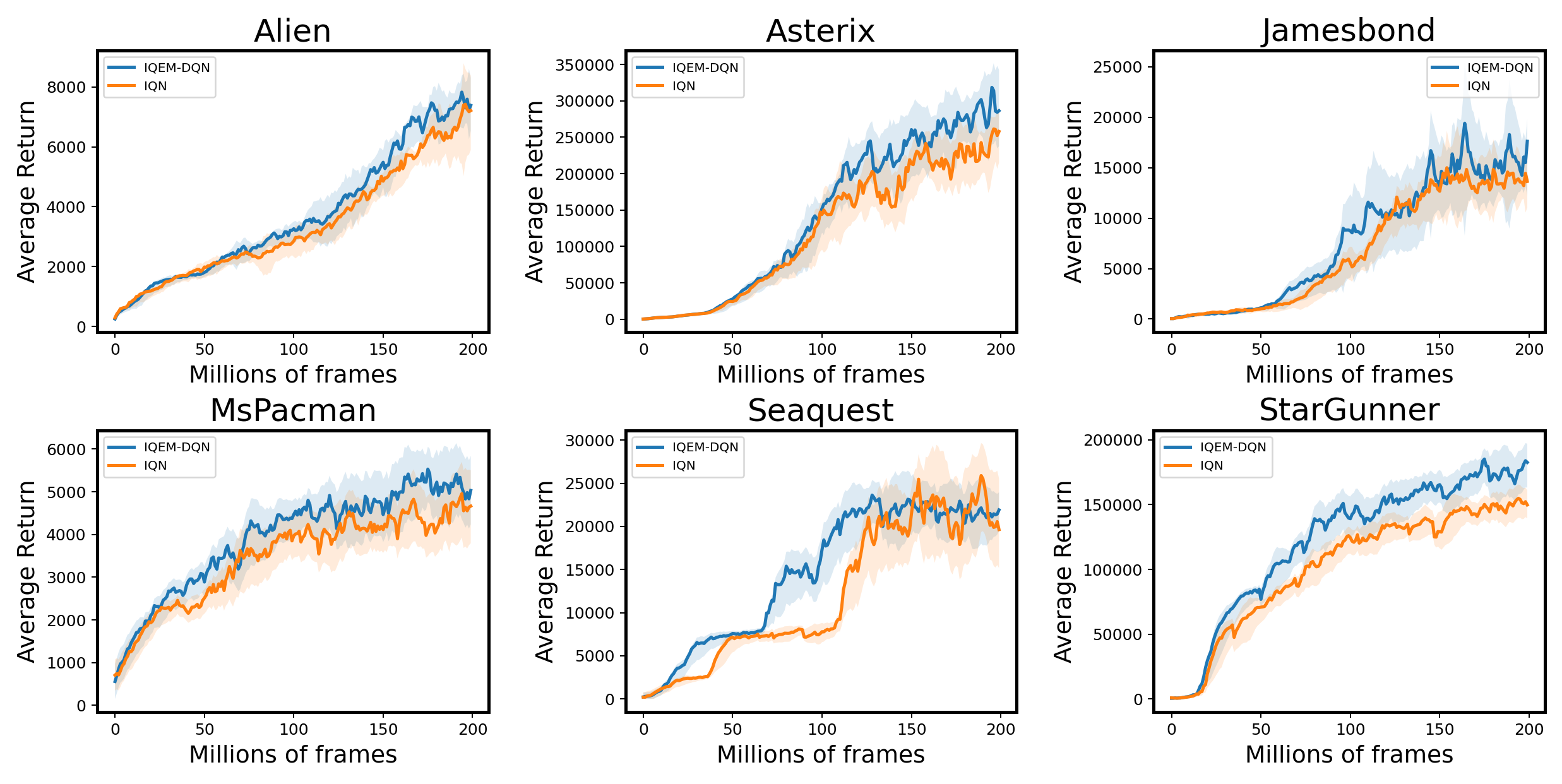}
\vskip -0.1in
\caption{Learning curves of IQN and IQEM-DQN across six Atari games. Each curve is averaged over 3 random seeds and shaded by their confidence intervals.}
\label{iqcm}
\end{figure}

%\subsection{Evaluation on Atari 2600} 
%We also extend QCM to IQN and obtain similar experimental results as shown in \cref{iqcm}.

\begin{figure}[!h]
\includegraphics[width = 0.99\linewidth]{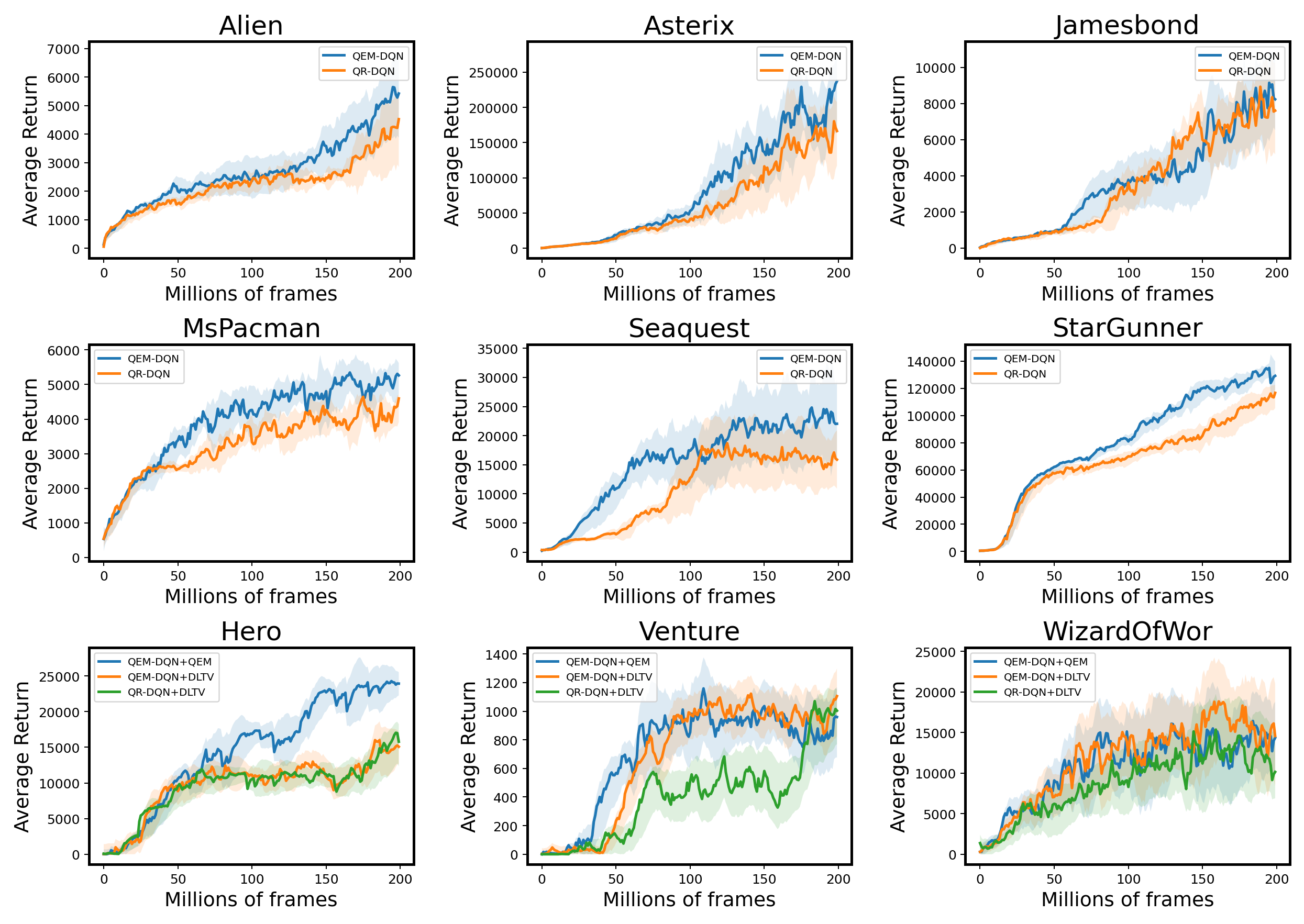}
\vskip -0.1in
\caption{Learning curves (top and middle) of QR-DQN and QEM-DQN across six Atari games. Learning curves (bottom) of QR-DQN and QEM-DQN with exploration across three games.}
\label{atari}
\vskip -0.1in
\end{figure}

\subsection{Exploration}
Since QEM also provides an estimate of the variance, we may consider using it to develop an efficient exploration strategy. 
%which can be used to develop exploration strategies. 
In some recent study studies, to more sufficiently utilize the distribution information, \citet{DLTV} proposes a novel exploration strategy, Decaying Left Truncated Variance (DLTV) by using the left truncated variance of the estimated distribution as a bonus term to encourage exploration in unknown states. The optimal action $a^{*}$ at state $x$ is selected according to $a^{*}=\arg \max _{a^{\prime}}\left(Q\left(x, a^{\prime}\right)+c_{t} \sqrt{\sigma^{2}_{+}}\right)$,
where $c_{t}$ is a decay factor to suppress the intrinsic uncertainty, and $\sigma^{2}_{+}$ denotes the estimation of variance.  Although DLTV is effective, the validity of the computed truncation lacks a theoretical guarantee. In this work, we follow the idea of DLTV and examine the model performance by using either the variance estimate obtained by QEM or the original DLTV estimation in some hard-explored games. As \cref{atari} shows, by using QEM, the exploration efficiency is significantly improved compared to QR-DQN+DLTV since QEM enhances the accuracy of the quantile estimates and thus the accuracy of the distribution variance.  
%adopted this UCB-style exploration strategy and compared our QCM variance estimation with DLTV estimation across several hard exploration games shown in \cref{atari}.  By applying the two bonus terms to QCM-DQN, exploration is significantly improved over QR-DQN+DLTV, since it enhances the accuracy of the distribution estimation.  
%Overall, the variance estimate of QCM is more efficient and performs better than DLTV.

\section{Conclusion and Discussion}
In this work, we systematically study the three error terms associated with the Q-function estimate and propose a novel DRL algorithm QEMRL, which can be applied to any quantile-based DRL algorithm regardless of whether the quantile locations are fixed or not. We found that a more robust estimate of the Q-function can improve the distribution approximation and speed up the algorithm convergence. 
%QCM is versatile and can be applied to any quantile-based DRL algorithm, .  
%In addition, w\
We can also utilize the more precise estimate of the distribution variance to optimize the existing exploration strategy.

Finally, there are some open questions we would like to have further discussions here.

\textbf{Improving the estimation of weight matrix $V$.}
The challenge of estimating the weight matrix $V$ was recognized from the outset of the method proposal since it is unlikely to know the exact value of $V$ in practice.
In this work, we treat $V$ as a predefined value that can be tuned, taking into account the computational cost of estimating it across all state-action pairs and time steps. As for future work, we believe a robust and easy-to-implement estimation of weight matrix $V$ is necessary. Given that the variance of quantile estimation errors varies with state-action pairs and algorithm iterations, we consider two approaches for future investigation. The first approach considers a decay value of $v_i$ instead of the constant. It is worth noting that the variance of poorly estimated quantiles tends to decrease gradually as the number of training samples increases, which motivates us to decrease the value of $v_i$ as training epochs increase. The second approach involves assigning different values of $v_i$ to different state-action pairs. Ideas from the exploration field, specifically the count-based method \cite{ostrovski2017count}, can be borrowed to measure the novelty of state-action pairs. Accordingly, for familiar state-action pairs, a smaller value of $v_i$ should be assigned, while unfamiliar pairs should be assigned a larger value of $v_i$.

\textbf{Statistical variance reduction.}
Our variance reduction method is based on a statistical modeling perspective, and the core insight of our method is that performance might be improved through more careful use of the quantiles to construct a Q-function estimator. While alternative ensembling methods can be directly applied to DRL to reduce the uncertainty in Q-function estimator, commonly used in existing works \cite{osband2016,AVERAGED-DQN}, it undoubtedly increases model complexity. In this work, we transform the Q value estimation into a linear regression problem, where the Q value is the coefficient of the regression model. In this way, we can leverage the weighted least squares (WLS) method to effectively capture the heteroscedasticity of quantiles and obtain a more efficient and robust Q-function estimator.

%\section*{Accessibility}
%Authors are kindly asked to make their submissions as accessible as possible for everyone including people with disabilities and sensory or neurological differences.
%Tips of how to achieve this and what to pay attention to will be provided on the conference website \url{http://icml.cc/}.

%\section*{Software and Data}

%If a paper is accepted, we strongly encourage the publication of software and data with the camera-ready version of the paper whenever appropriate. This can be done by including a URL in the camera-ready copy. However, \textbf{do not} include URLs that reveal your institution or identity in your submission for review. Instead, provide an anonymous URL or upload the material as ``Supplementary Material'' into the CMT reviewing system. Note that reviewers are not required to look at this material when writing their review.

\section*{Acknowledgements}

We thank anonymous reviewers for valuable and constructive feedback on an early version of this manuscript. This work is supported by National Social Science Foundation of China (Grant No.22BTJ031 ) and Postgraduate Innovation Foundation of SUFE. Dr. Fan Zhou's work is supported by National Natural Science Foundation of China (12001356), Shanghai Sailing Program (20YF1412300), “Chenguang Program” supported by Shanghai Education Development Foundation and Shanghai Municipal Education Commission, Open Research Projects of Zhejiang Lab (NO.2022RC0AB06), Shanghai Research Center for Data Science and Decision Technology, Innovative Research Team of Shanghai University of Finance and Economics.


\bibliography{reference}
\bibliographystyle{icml2023}

%%%%%%%%%%%%%%%%%%%%%%%%%%%%%%%%%%%%%%%%%%%%%%%%%%%%%%%%%%%%%%%%%%%%%%%%%%%%%%%
%%%%%%%%%%%%%%%%%%%%%%%%%%%%%%%%%%%%%%%%%%%%%%%%%%%%%%%%%%%%%%%%%%%%%%%%%%%%%%%
% APPENDIX
%%%%%%%%%%%%%%%%%%%%%%%%%%%%%%%%%%%%%%%%%%%%%%%%%%%%%%%%%%%%%%%%%%%%%%%%%%%%%%%
%%%%%%%%%%%%%%%%%%%%%%%%%%%%%%%%%%%%%%%%%%%%%%%%%%%%%%%%%%%%%%%%%%%%%%%%%%%%%%%
%%%%%%%%%%%%%%%%%%%%%%%%%%%%%%%%%%%%%%%%%%%%%%%%%%%%%%%%%%%%%%%%%%%%%%%%%%%%%%%
%%%%%%%%%%%%%%%%%%%%%%%%%%%%%%%%%%%%%%%%%%%%%%%%%%%%%%%%%%%%%%%%%%%%%%%%%%%%%%%
% APPENDIX
%%%%%%%%%%%%%%%%%%%%%%%%%%%%%%%%%%%%%%%%%%%%%%%%%%%%%%%%%%%%%%%%%%%%%%%%%%%%%%%
%%%%%%%%%%%%%%%%%%%%%%%%%%%%%%%%%%%%%%%%%%%%%%%%%%%%%%%%%%%%%%%%%%%%%%%%%%%%%%%
\newpage
\appendix
\onecolumn

\section{Projection Operator}
\subsection{Categorical projection operator}
 CDRL algorithm uses a categorical projection operator $\Pi_{\mathcal{C}}:\mathscr{P}(\mathbb{R})\to \mathscr{P}\left(\left\{z_{1}, \ldots, z_{N}\right\}\right)$ to restrict approximated distributions to the parametric family of the form $\mathscr{F}_{\mathcal{C}}:=\left\{\sum_{i=1}^{N} p_{i} \delta_{z_{i}} \mid \sum_{i=1}^{N} p_{i}=1, p_{i} \geq 0 \right\} \subseteq \mathscr{P}(\mathbb{R})$, where $z_{1}<\cdots<z_{N}$ are evenly spaced, fixed supports. The operator $\Pi_{\mathcal{C}}$ is defined for a single Dirac delta as
$$
\Pi_{\mathcal{C}}\left(\delta_{w}\right)= \begin{cases}\delta_{z_{1}} & w \leq z_{1} \\ \frac{w-z_{i+1}}{z_{i}-z_{i+1}} \delta_{z_{i}}+\frac{z_{i}-w}{z_{i}-z_{i+1}} \delta_{i+1} & z_{i} \leq w \leq z_{i+1} \\ \delta_{z_{N}} & w \geq z_{N}.\end{cases}
$$

\subsection{Quantile projection operator}

QDRL algorithm uses a quantile projection operator $\Pi_{\mathcal{W}_{1}}:\mathscr{P}(\mathbb{R})\to \mathscr{P}(\mathbb{R})$ to restrict approximated distributions to the parametric family of the form $\mathscr{F}_{\mathcal{W}_{1}}:=\left\{\frac{1}{N} \sum_{i=1}^{N} \delta_{z_{i}} \mid z_{1: N} \in \mathbb{R}^{N}\right\} \subseteq \mathscr{P}(\mathbb{R})$. The operator $\Pi_{W_{1}}$ is defined as
$$
\Pi_{\mathcal{W}_{1}}(\mu)=\frac{1}{N} \sum_{k=1}^{N} \delta_{F_{\mu}^{-1}\left(\tau_{i}\right)},
$$
where $\tau_{i}=\frac{2 i-1}{2 N}$, and $F_{\mu}$ is the CDF of $\mu$. The midpoint $\frac{2 i-1}{2 N}$ of the interval $[\frac{i-1}{N},\frac{i}{N}]$ minimizes the 1-Wasserstein distance $W_{1}(\mu, \Pi_{W_{1}}{\mu})$ between the distribution, $\mu$, and its projection $\Pi_{W_{1}}{\mu}$ (a $N$-quantile distribution with evenly spaced $\tau_{i}$), as demonstrated in Lemma 2 \cite{QRDQN}.

%The choice of the midpoint at $\frac{2 i-1}{2 N}$ level minimizes the 1-Wasserstein distance $W_{1}(\mu, \Pi_{W_{1}}{\mu})$ on the interval with evenly spaced $\tau_{i}$, as has been demenstrated in Lemma 2  \cite{QRDQN}.

\section{Proofs}\label{proof}
In this section, we provide the proofs of the theorems discussed in the main manuscript.

\subsection{Proof of \cref{sec3}}
\begin{proposition}[\citealp{sobel1982,C51}]
Suppose there are value distributions $\nu_{1}, \nu_{2} \in \mathscr{P}(\mathbb{R})$, and random variables $Z_{i}^{k+1}\sim \mathcal{T}^{\pi} \nu_{i}, Z_{i}^{k}\sim \nu_{i} $. Then, we have
\begin{align*}
\left\|\mathbb{E}Z_{1}^{k+1}-\mathbb{E}Z_{2}^{k+1}\right\|_{\infty} \leq \gamma\left\|\mathbb{E}Z_{1}^{k}-\mathbb{E}Z_{2}^{k}\right\|_{\infty}, \text { and } \\
\left\|\mathrm{Var}Z_{1}^{k+1}-\mathrm{Var}Z_{2}^{k+1}\right\|_{\infty} \leq \gamma^{2}\left\|\mathrm{Var}Z_{1}^{k}-\mathrm{Var}Z_{2}^{k}\right\|_{\infty}.
\end{align*}
%where $\|\cdot\|_\infty$ is the maximal form of the Wasserstein metric.
\end{proposition}

\begin{proof} 
This proof follows directly from \citet{C51}. The first statement can be proved using the exchange of $\mathbb{E} \mathcal{T}^{\pi} =\mathcal{T}^{\pi} \mathbb{E}$. By independence of $R$ and $P^{\pi} Z_{i}$, where $P^{\pi}$ is the transition operator, we have
\begin{align*}
Z_{i}^{k+1}(x, a) &\stackrel{D}{:=}R(x, a)+\gamma P^{\pi} Z_{i}^{k}(x, a) \\
\mathrm{Var}(Z_{i}^{k+1}(x, a))&=\mathrm{Var}(R(x, a))+\gamma^{2} \mathrm{Var}\left(P^{\pi} Z_{i}^{k}(x, a)\right) .
\end{align*}
Thus, we have
\begin{align*}
\left\|\mathrm{Var}Z_{1}^{k+1}-\mathrm{Var}Z_{2}^{k+1}\right\|_{\infty} &\\
&=\sup _{x, a}\left|\mathrm{Var}Z_{1}^{k+1}(x,a)-\mathrm{Var}Z_{2}^{k+1}(x,a)\right| \\
&=\sup _{x, a} \gamma^{2}\left|\mathrm{Var}\left(P^{\pi} Z_{1}^{k}(x, a)\right)-\mathrm{Var}\left(P^{\pi} Z_{2}^{k}(x, a)\right)\right| \\
&=\sup _{x, a} \gamma^{2}\left|\mathbb{E}\left[\mathrm{Var}\left(Z_{1}^{k}\left(X^{\prime}, A^{\prime}\right)\right)-\mathrm{Var}\left(Z_{2}^{k}\left(X^{\prime}, A^{\prime}\right)\right)\right]\right| \\
& \leq \sup _{x^{\prime}, a^{\prime}} \gamma^{2}\left|\mathrm{Var}\left(Z_{1}^{k}\left(x^{\prime}, a^{\prime}\right)\right)-\mathrm{Var}\left(Z_{2}^{k}\left(x^{\prime}, a^{\prime}\right)\right)\right| \\
& \leq \gamma^{2}\left\|\mathrm{Var} Z_{1}^{k}-\mathrm{Var} Z_{2}^{k}\right\|_{\infty} .
\end{align*}
\end{proof}

%\begin{lemma}(non-expansion)
%Consider a representation $\mathscr{F} \subseteq \mathscr{P}_{\ell_{2}}(\mathbb{R})$. If $\mathscr{F}$ is complete with respect to $\ell_{2}$ and convex, then the $\ell_{2}$-projection $\Pi \mathscr{F}, \ell_{2}: \mathscr{P}_{\ell_{2}}(\mathbb{R}) \rightarrow \mathscr{F}$ is a nonexpansion in that metric:
%$$
%\left\|\Pi \mathscr{F}_{,} \ell_{2}\right\|_{\ell_{2}}=1 .
%$$
%Furthermore, the result extends to return functions and the supremum extension of $\ell_{2}$ :
%$$
%\left\|\Pi \Pi_{\mathscr{F}, \ell_{2}}\right\|_{\bar{\ell}_{2}}=1 .
%$$
%\end{lemma}

\begin{lemma}[Lemma B.2 of \citet{EDRL}]\label{lemmawdistance}
Let $\tau_{k}=\frac{2 k-1}{2 K}$, for $k=1, \ldots, K$. Consider the corresponding 1-Wasserstein projection operator $\Pi_{W_{1}}: \mathscr{P}(\mathbb{R}) \rightarrow \mathscr{P}(\mathbb{R})$, defined by
$$
\Pi_{W_{1}}\mu_{i}=\frac{1}{K} \sum_{k=1}^{K} \delta_{F_{\mu_{i}}^{-1}\left(\tau_{k}\right)},
$$
for all $\mu_{i} \in \mathscr{P}(\mathbb{R})$, where $F_{\mu_{i}}^{-1}$ is the inverse CDF of $\mu_{i}$. Let random variable $X \sim \mu_{1}$, $X^{2}\sim \mu_{2}$, and $\eta_{1}, \eta_{2} \in \mathscr{P}(\mathbb{R})$. Suppose immediate reward distributions supported on $[- R_{max}, R_{max}]$. Then, we have:\\
(i) $W_{1}\left(\Pi_{W_{1}} \mu_{1}, \mu_{1}\right) \leq \frac{2 R_{max}}{K(1-\gamma)}$;\\
(ii) $W_{1}\left(\Pi_{W_{1}} \eta_{1}, \Pi_{W_{1}} \eta_{2}\right) \leq W_{1}\left(\eta_{1}, \eta_{2}\right)+\frac{4 R_{max}}{K(1-\gamma)}$;\\
(iii) $W_{1}\left(\Pi_{W_{1}} \mu_{2}, \mu_{2}\right) \leq \frac{R_{max}^{2}}{K(1-\gamma)}$.
\end{lemma}

\begin{proof}
This proof follows directly from Lemma B.2 of \citet{EDRL}. For proving (i), let $F_{\mu_{1}}^{-1}$ be the inverse CDF of $\mu_{1}$. We have
\begin{align*}
W_{1}\left(\Pi_{W_{1}} \mu_{1},\mu_{1}\right) &=\sum_{i=0}^{K-1} \frac{1}{K} \int_{F_{\mu_{1}}^{-1}(\frac{i}{K})}^{F_{\mu_{1}}^{-1}(\frac{i+1}{K})}|x-F_{\mu_{1}}^{-1}(\frac{2 i+1}{2 K})|\quad\mu_{1}(dx) \\
& \leq \frac{1}{K}\left(F_{\mu_{1}}^{-1}(1)-F_{\mu_{1}}^{-1}(0)\right) \quad (\text{return distribution $\mu_{1}$ is bounded on} ~ [-\frac{R_{max}}{1-\gamma},\frac{R_{max}}{1-\gamma}]) \\
&=\frac{2 R_{max}}{K (1-\gamma)}.
\end{align*}
For proving (ii), using the triangle inequality and statement (i):
\begin{align*}
W_{1}\left(\Pi_{W_{1}} \eta_{1}, \Pi_{W_{1}} \eta_{2}\right) & \leq W_{1}\left(\Pi_{W_{1}} \eta_{1}, \eta_{1}\right)+W_{1}\left(\eta_{1}, \eta_{2}\right)+W_{1}\left(\eta_{2}, \Pi_{W_{1}} \eta_{2}\right) \\
& \leq W_{1}\left(\eta_{1}, \eta_{2}\right)+\frac{4 R_{max}}{K(1-\gamma)}.
\end{align*}
(ii) implies the fact that the quantile projection operator $\Pi_{W_{1}}$ is not a non-expansion under 1-Wasserstein distance, which is important for the uniqueness of the fixed point and the convergence of the algorithm.

The proof of (iii) is similar to (i), using the fact that the return distribution $\mu_{2}$ is bounded on $[0,\frac{R_{max}^{2}}{1-\gamma}]$ to obtain the following inequality:
\begin{align*}
W_{1}\left(\Pi_{W_{1}} \mu_{2}, \mu_{2}\right) \leq \frac{R_{max}^{2}}{K(1-\gamma)}.
\end{align*}

\end{proof}

%\begin{theorem}
%(\textbf{Parameterization induced error bound}) Let $\Pi_{\mathcal{W}_{1}}$ be the projection operator onto the evenly spaced $\tau_{i}$ th quantile, $\tau_{i}=\frac{2 i-1}{2 N},$ for $i=1, \ldots, N$, and $\eta_{k}\in\mathscr{P}(\mathbb{R})$ be the return distribution of k th iteration. Let random variables $Z_{\theta}^{k}\sim\Pi_{\mathcal{W}_{1}}\mathcal{T}^{\pi}\eta_{k},Z^{k}\sim\mathcal{T}^{\pi}\eta_{k}$. Suppose immediate reward distributions supported on $[- R_{max}, R_{max}]$, then
%\begin{align*}
%\lim _{k \rightarrow \infty}\left\|\mathbb{E}Z_{\theta}^{k}-%\mathbb{E}Z^{k}\right\|_{\infty} \leq \frac{2 R_{max}}{N (1-\gamma)},\\
%\lim _{k \rightarrow \infty}\left\|\mathbb{E}[Z_{\theta}^{k}]^{2}-\mathbb{E}[Z^{k}]^{2}\right\|_{\infty} \leq \frac{R_{max}^{2}}{N (1-\gamma)}.
%\end{align*}
%\end{theorem}
\begin{theorem} [\textbf{Parameterization induced error bound}] Let $\Pi_{\mathcal{W}_{1}}$ be a projection operator onto evenly spaced quantiles $\tau_{i}$'s where each 
%$\tau_{i}$ th quantile, 
$\tau_{i}=\frac{2 i-1}{2 N}$ for $i=1, \ldots, N$, and $\eta_{k}\in\mathscr{P}(\mathbb{R})$ be the return distribution of $k$-th iteration. Let random variables $Z_{\theta}^{k}\sim\Pi_{\mathcal{W}_{1}}\mathcal{T}^{\pi}\eta_{k}$ and $Z^{k}\sim\mathcal{T}^{\pi}\eta_{k}$. Assume that the distribution of the immediate reward is supported on $[- R_{max}, R_{max}]$, then we have
\begin{align*}
\lim _{k \rightarrow \infty} \left\|\mathcal{E}^k_3 \right\|_{\infty} = \lim _{k \rightarrow \infty}\left\|\mathbb{E}Z_{\theta}^{k}-\mathbb{E}Z^{k}\right\|_{\infty} \leq \frac{2 R_{max}}{N (1-\gamma)},
%\lim _{k \rightarrow \infty}\left\|\mathbb{E}[Z_{\theta}^{k}]^{2}-\mathbb{E}[Z^{k}]^{2}\right\|_{\infty} \leq \frac{R_{max}^{2}}{N (1-\gamma)}.
\end{align*}
where $\mathcal{E}^k_3$ is parametrization induced error at $k$-th iteration.
\end{theorem}

\begin{proof}
Using the dual representation of the Wasserstein distance \cite{villani2009optimal} and \cref{lemmawdistance}, $\forall (x, a)$, we have
\begin{align*}
\left|\mathbb{E}Z_{\theta}^{k}(x,a)-\mathbb{E}Z^{k}(x,a)\right| & \leq W_{1}\left(\Pi_{W_{1}}\mathcal{T}^{\pi}\eta_{k}(x, a),\mathcal{T}^{\pi}\eta_{k}(x, a)\right) \\
& \leq \frac{2 R_{max}}{N (1-\gamma)}. 
\end{align*}
By taking the limitation over $(x, a)$ and iteration $k$ on the left-hand side, we obtain
$$
\lim _{k \rightarrow \infty} \left\|\mathcal{E}^k_3 \right\|_{\infty}=\lim _{k \rightarrow \infty}\left\|\mathbb{E}Z_{\theta}^{k}-\mathbb{E}Z^{k}\right\|_{\infty} \leq \frac{2 R_{max}}{N (1-\gamma)}.
$$
In a similar way, the second-order moment can be bounded by,
$$
\lim _{k \rightarrow \infty}\left\|\mathbb{E}[Z_{\theta}^{k}]^{2}-\mathbb{E}[Z^{k}]^{2}\right\|_{\infty} \leq \frac{R_{max}^{2}}{N (1-\gamma)}.
$$
It suggests that higher-order moments are not preserved after quantile representation is applied.
\end{proof}

\subsection{Proof of \cref{sec4}}

\begin{lemma} [expectation by quantiles]. Let $Z\sim\nu$ be a random variable with CDF $F_{\nu}$ and quantile function $F_{\nu}^{-1}$. Then,
$$
\mathbb{E}[Z]=\int_{0}^{1} F_{\nu}^{-1}(\tau) d \tau .
$$
\end{lemma}

\begin{proof}

As any CDF is non-decreasing and right continuous, we have for all $(\tau, z) \in(0,1) \times \mathbb{R}$ :
$$
F_{\nu}^{-1}(\tau) \leq z \Longleftrightarrow \tau \leq F_{\nu}(z) .
$$
Then, denoting $U$ by a uniformly distributed random variable over $[0,1]$,
$$
\mathbb{P}(F_{\nu}^{-1}(U) \leq z)=\mathbb{P}(U \leq F_{\nu}(z))=F_{\nu}(z),
$$
which shows that the random variable $F_{\nu}^{-1}(U)$ has the same distribution as $Z$. Hence,
$$
\mathbb{E}[Z]=\mathbb{E}\left[F_{\nu}^{-1}(U)\right]=\int_{0}^{1} F_{\nu}^{-1}(\tau) d \tau
$$
\end{proof}

\begin{lemma}
Consider the linear regression model $\boldsymbol{\hat{Q}}=\mathbf{X}_2 \boldsymbol{M}_2+\mathcal{E}$,  $\mathcal{E}$ is distributed on $\mathcal{N}(\mathbf{0},\sigma^{2} V)$, where $V=diag(v_{1},v_{2},\cdots,v_{N}),v_{i}\geq 1, i=1,\cdots,N$, and we set noise variance  $\sigma^2 = 1$ without loss of generality. 
The WLS estimator is
\begin{align}
\widehat{\boldsymbol{M}}_2 = (\mathbf{X}^{\top}_2 V^{-1} \mathbf{X}_2)^{-1}\mathbf{X}^{\top}_2 V^{-1} \boldsymbol{\hat{Q}},
\end{align}
and the distribution of mean estimator takes the form, $$\hat{m}_{1}\sim \mathcal{N}\left(m_{1},\frac{1}{\sum_{i}v_{i}}+\frac{(\frac{\sum_{i}v_{i}z_{\tau_i}}{\sum_{i}v_{i}})^2}{\sum_{i}v_{i}z_{\tau_i}^2-\frac{(\sum_{i}v_{i}z_{\tau_i})^2}{\sum_{i}v_{i}}}\right).$$  When $V$ equals identity matrix $I$, $$\hat{m}_{1}\sim \mathcal{N}\left(m_{1},\frac{1}{N}+\frac{\bar{z}^2}{\sum_{i}(z_{\tau_i}-\bar{z})^2}\right).$$
\end{lemma}

\begin{proof}
Premultiplying by $V^{-1 / 2}$, we get the transformed model
$$
V^{-1 / 2} \boldsymbol{\hat{Q}}=V^{-1 / 2} \mathbf{X}_2 \boldsymbol{M}_2+V^{-1 / 2} \mathcal{E}.
$$
Now, set $\boldsymbol{\hat{Q}}^{*}=V^{-1 / 2} \boldsymbol{Q}, X^{*}_2=V^{-1 / 2} X_2$, and $\mathcal{E}^{*}=V^{-1 / 2} \mathcal{E}$, so that the transformed model can be written as $\boldsymbol{\hat{Q}}^{*}=\mathbf{X}^{*}_2 \boldsymbol{M}_2+\mathcal{E}^{*}$. The transformed model is a Gaussian-Markov model,  satisfying OLS assumptions. Thus, the unique OLS solution is $\widehat{\boldsymbol{M}}_2=\left(X^{\top}_2 V^{-1} X_2\right)^{-1} X^{\top}_2 V^{-1} \boldsymbol{\hat{Q}},$ and $\widehat{\boldsymbol{M}}_2 \sim \mathcal{N}\left(\boldsymbol{M}_2,\sigma^{2}(X_2^{\top} V^{-1} X_2)^{-1}\right).$ By computing $(X^{\top}_2 V^{-1} X_2)^{-1},$ we derive $\hat{m}_{1}\sim \mathcal{N}\left(m_{1},\frac{1}{\sum_{i}v_{i}}+\frac{(\frac{\sum_{i}v_{i}z_{\tau_i}}{\sum_{i}v_{i}})^2}{\sum_{i}v_{i}z_{\tau_i}^2-\frac{(\sum_{i}v_{i}z_{\tau_i})^2}{\sum_{i}v_{i}}}\right).$
\end{proof}

%\begin{proposition}
%Suppose the noise $\varepsilon_{i}$ independently follows $\mathcal{N}(0, v_{i})$, $v_{i}\geq 1$, $i=1,\cdots,N$, then, $\mathrm{Var}(\hat{m}_{1}^{EM})=\frac{\sum_i v_i}{N^2}$. (i) In the homoskedastic case, i.e. $v_{i}=1$, the emprical mean estimator $\hat{m}_{1}^{EM}$ has a lower variance, $\mathrm{Var}(\hat{m}_{1}^{EM})<\mathrm{Var}(\hat{m}_{1}^{QCM})$ ; (ii) In the heteroskedastic case, i.e. $v_{i}$ are not eaqul, the QCM estimator $\hat{m}_{1}^{QCM}$ achieves a lower variance, $\mathrm{Var}(\hat{m}_{1}^{QCM})<\mathrm{Var}(\hat{m}_{1}^{EM})$, if $\bar{v}^2-1-1/(\frac{(\sum_{i}v_{i}\sum_{i}v_{i}z_{i}^2)}{(\sum_{i}v_{i}z_{i})^2}-1)>0$, where $\bar{v}=\frac{1}{N}\sum_{i}v_{i}$, and the inequality holds when $v_i> 1$ and $z_i = -z_{N-i}$. 
%\end{proposition}
\begin{proposition}
Suppose the noise $\varepsilon_{i}$ independently follows $\mathcal{N}(0, v_{i})$ where $v_{i}\geq 1$ for $i=1,\cdots,N$, then, %$\mathrm{Var}(\hat{m}^{*}_{1})=\frac{\sum_i v_i}{N^2}$. 

(i) In the homoskedastic case where $v_{i}=1$ for $i = 1, \dots N$, the empirical mean estimator $\hat{m}_{1}^{*}$ has a lower variance, $\mathrm{Var}(\hat{m}_{1}^{*})<\mathrm{Var}(\hat{m}_{1})$ ; 

(ii) In the heteroskedastic case where $v_{i}$'s are not equal, the QEM estimator $\hat{m}_{1}$ achieves a lower variance, i.e. $\mathrm{Var}(\hat{m}_{1})<\mathrm{Var}(\hat{m}_{1}^{*})$, if and only if $\bar{v}^2-1-1/(\frac{(\sum_{i}v_{i}\sum_{i}v_{i}z_{\tau_i}^2)}{(\sum_{i}v_{i}z_{\tau_i})^2}-1)>0$, where $\bar{v}=\frac{1}{N}\sum_{i}v_{i}$. This inequality holds when $z_{\tau_i} = -z_{\tau_{N-i}}$, which can be guaranteed in QDRL. 

\end{proposition}

\begin{proof}
The proof of (i) comes directly from the comparison of variances, i.e. $\mathrm{Var}(\hat{m}_{1})=\frac{1}{N}<\frac{1}{N}+\frac{\bar{z}^2}{\sum_{i}(z_{\tau_i}-\bar{z})^2}=\mathrm{Var}(\hat{m}_{1}^{*})$. Next, we prove that (ii) holds under a sufficient condition $z_{\tau_i} =-z_{\tau_{N-i}}$. In QDRL, the quantile levels $\tau_{i}=\frac{2i-1}{2N}$ are equally spaced around 0.5. Under this setup, the condition $z_{\tau_i} = -z_{\tau_{N-i}}$ indeed holds, where $z_{\tau_i}$ is the $\tau_i$-th quantile of standard normal distribution. For $N=2$, we need to validate the inequality $\bar{v}^2-1-1/(\frac{(\sum_{i}v_{i}\sum_{i}v_{i}z_{\tau_i}^2)}{(\sum_{i}v_{i}z_{\tau_i})^2}-1)>0$. This can be transformed into a multivariate extreme value problem. By analyzing the function $f(v_1,v_2)=\frac{(v_1+v_2)^2}{4}-1-\frac{1}{\frac{(v_1+v_2)^2}{(v_1-v_2)^2}-1}$, the infimum of $f(v_1,v_2)$ is 0 when $v_1,v_2>1$, and $f(v_1,v_2)$ reaches 0 at the limit $\lim_{(v_1,v_2)\to (1,1)}f(v_1,v_2)=0$. For $N=3$, this case is identical to $N=2$ since $z_{0.5}=0$. For $N=4$, $f(v_1,v_2,v_3,v_4)=\frac{(v_1+v_2+v_3+v_4)^2}{N^2}-1-\frac{1}{\frac{(v_1+v_2+v_3+v_4)(k^2 v_1+v_2+v_3+k^2 v_4)}{(kv_1+v_2-v_3-kv_4)^2}-1}$, and this expression can be factored as, $f(v_1,v_2,v_3,v_4)=\frac{v_1+v_2+v_3+v_4}{N^2C}\left((v_1+v_2+v_3+v_4)C-N^2(k^2 v_1+v_2+v_3+k^2 v_4)\right)$, where $C=(k-1)^2v_1v_2+(k+1)^2v_1v_3+4k^2v_1v_4+4v_2v_3+(k+1)^2v_2v_4+(k+1)^2v_3v_4$, and $k=\frac{\Phi^{-1}(7/8)}{\Phi^{-1}(5/8)}>3$. By comparing the coefficient corresponding to the same terms, we can verify that $f(v_1,v_2,v_3,v_4)>0$ when $v_i>1$. Finally, the remaining cases can be proven in the same manner.

%since $k=\Phi^{-1}(\frac{2N-1}{2N})/\Phi^{-1}(\frac{2(\lceil\frac{N+1}{2}\rceil)-1}{2N})>N$ when $N>4$.

\end{proof}

%To be more specific, we provide another proof and relax the condition $z_{\tau_i} = -z_{\tau_{N-i}}$ to $\sum z_{\tau_i} = 0$. By Cauchy inequality, $(\sum_{i}v_{i}\sum_{i}v_{i}z_{\tau_i}^2)=(\sum_{i}\sqrt{v_{i}}^2\sum_{i}(\sqrt{v_{i}}z_{\tau_i})^2)\geq(\sum_{i}v_{i}z_{\tau_i})^2$. Then, $1/(\frac{(\sum_{i}v_{i}\sum_{i}v_{i}z_{\tau_i}^2)}{(\sum_{i}v_{i}z_{\tau_i})^2}-1)=$. 

\begin{theorem}%\label{theorem:concentration}
Consider the policy $\hat{\pi}$ that is learned policy, and
denote the optimal policy to be $\pi_{opt}$, $\alpha=\max_{x'} D_{TV}(\hat{\pi}\left(\cdot\mid x^{\prime})\| \pi_{opt}(\cdot \mid x^{\prime})\right)$, and $n(x,a)=|\mathcal{D}|$. For all $\delta \in \mathbb{R}$, with probability at least $1-\delta$, for any $\eta(x,a) \in \mathscr{P}(\mathbb{R})$, and all $(x, a) \in \mathcal{D}$,
\begin{align*}
\left\|F_{\hat{\mathcal{T}}^{\hat{\pi}} \eta(x, a)}-F_{\mathcal{T}^{\pi_{opt}} \eta(x, a)}\right\|_{\infty}\leq 2\alpha + \sqrt{\frac{1+4|\mathcal{X}|}{n(x, a)} \log \frac{4|\mathcal{X}||\mathcal{A}|}{\delta}}.
\end{align*}
\end{theorem}

\begin{proof}
We give this proof in a tabular MDP. Directly following from the definition of the distributional Bellman operator applied to the CDF, we have that
\begin{align*}
&F_{\hat{\mathcal{T}}^{\hat{\pi}}\eta(x, a)}(u)-F_{\mathcal{T}^{\pi_{opt}}\eta(x, a)}(u) \\
&=\sum_{x^{\prime}, a^{\prime}} \hat{P}(x^{\prime} \mid x, a) \hat{\pi}(a^{\prime} \mid x^{\prime}) F_{\gamma Z\left(x^{\prime}, a^{\prime}\right)+\hat{R}(x, a)}(u)-\sum_{x^{\prime}, a^{\prime}} P(x^{\prime} \mid x, a) \pi_{opt}(a^{\prime} \mid x^{\prime}) F_{\gamma Z(x^{\prime}, a^{\prime})+R(x, a)}(u) .
\end{align*}
For notation convenience, we use random variables instead of measures. $\hat{P}$ and $\hat{R}$ are the maximum likelihood estimates of the transition and the reward functions, respectively. Adding and subtracting $\sum_{x^{\prime}, a^{\prime}} \hat{P}(x^{\prime} \mid x, a) \pi_{opt}(a^{\prime} \mid x^{\prime}) F_{\gamma Z(x^{\prime}, a^{\prime})+R(x, a)}(u)$, then we have
\begin{align*}
&\sum_{x^{\prime}} \hat{P}(x^{\prime} \mid x, a) \sum_{a^{\prime}} \left(\hat{\pi}(a^{\prime} \mid x^{\prime}) F_{\gamma Z(x^{\prime}, a^{\prime})+\hat{R}(x, a)}(u)- \pi_{opt}(a^{\prime} \mid x^{\prime})F_{\gamma Z(x^{\prime}, a^{\prime})+R(x, a)}(u)\right) \\
&+\sum_{x^{\prime}, a^{\prime}}\left(\hat{P}(x^{\prime} \mid x, a)-P(x^{\prime} \mid x, a)\right) \pi_{opt}(a^{\prime} \mid x^{\prime}) F_{\gamma Z(x^{\prime}, a^{\prime})+R(x, a)}(u) .
\end{align*}
For the first term, note that
\begin{align*}
&\sum_{x^{\prime}} \hat{P}(x^{\prime} \mid x, a) \sum_{a^{\prime}} \left(\hat{\pi}(a^{\prime} \mid x^{\prime}) F_{\gamma Z(x^{\prime}, a^{\prime})+\hat{R}(x, a)}(u)- \pi_{opt}(a^{\prime} \mid x^{\prime})F_{\gamma Z(x^{\prime}, a^{\prime})+R(x, a)}(u)\right)\\
&=\sum_{x'} \hat{P}(x' \mid x, a) \sum_{a'} \Big{(}\hat{\pi}(a'\mid x') F_{\gamma Z(x', a')+\hat{R}(x, a)}(u)-\hat{\pi}(a'\mid x') F_{\gamma Z(x', a')+R(x, a)}(u)   + \\
&\quad\quad\quad\quad\quad\quad\quad\quad\quad\quad\hat{\pi}(a^{\prime}\mid x^{\prime}) F_{\gamma Z(x^{\prime}, a^{\prime})+R(x, a)}(u)  - \pi_{opt}(a^{\prime} \mid x^{\prime})F_{\gamma Z(x^{\prime}, a^{\prime})+R(x, a)}(u)\Big{)}\\
%&\leq \sum_{x^{\prime}} \hat{P}(x^{\prime} \mid x, a) \sum_{a^{\prime}}\left|\hat{\pi}(a^{\prime}\mid x^{\prime})-\pi_{opt}(a^{\prime} \mid x^{\prime})\right| \cdot \left| F_{\gamma Z(x^{\prime}, a^{\prime})+\hat{R}(x, a)}(u)- F_{\gamma Z(x^{\prime}, a^{\prime})+R(x, a)}(u)\right|\\
&\leq\sum_{x^{\prime}} \hat{P}(x^{\prime} \mid x, a) \sum_{a^{\prime}} \Big{(} \left|\hat{\pi}(a^{\prime}\mid x^{\prime})\right| \cdot \left|F_{\gamma Z(x^{\prime}, a^{\prime})+\hat{R}(x, a)}(u)-F_{\gamma Z(x^{\prime}, a^{\prime})+R(x, a)}(u)\right| + \\
&\quad\quad\quad\quad\quad\quad\quad\quad\quad\quad \left|F_{\gamma Z(x^{\prime}, a^{\prime})+R(x, a)}(u)\right| \cdot \left|\hat{\pi}(a^{\prime}\mid x^{\prime})-\pi_{opt}(a^{\prime} \mid x^{\prime})\right| \Big{)}\\
&\stackrel{(a)}{\leq} \sum_{x^{\prime}} \hat{P}(x^{\prime} \mid x, a) \left( \left\|F_{\hat{R}(x, a)}(\cdot)-F_{R(x, a)}(\cdot)\right\|_{\infty} + 2 D_{TV}(\hat{\pi}\left(\cdot\mid x^{\prime})||\pi_{opt}(\cdot \mid x^{\prime})\right)\right)\\
& = 2\alpha +  \left\|F_{\hat{R}(x, a)}(\cdot)-F_{R(x, a)}(\cdot)\right\|_{\infty}.
\end{align*}
(a) follows from the fact that $\sum_{a^{\prime}} \left|\hat{\pi}(a^{\prime}\mid x^{\prime})-\pi_{opt}(\cdot \mid x^{\prime})\right| = 2 D_{TV}(\hat{\pi}\left(\cdot\mid x^{\prime})||\pi_{opt}(\cdot \mid x^{\prime})\right) \leq 2 \alpha$, and
\begin{align*}
& \sum_{a^{\prime}}  \left|\hat{\pi}(a^{\prime}\mid x^{\prime})\right| \cdot \left|F_{\gamma Z(x^{\prime}, a^{\prime})+\hat{R}(x, a)}(u)-F_{\gamma Z(x^{\prime}, a^{\prime})+R(x, a)}(u)\right|  \\
&=\sum_{a^{\prime}} \left|\hat{\pi}(a^{\prime}\mid x^{\prime})\right| \cdot
\int\left|F_{\hat{R}(x, a)}(r)-F_{R(x, a)}(r)\right| d F_{\gamma Z\left(x^{\prime}, a^{\prime}\right)}(u-r) \\
& \leq \sum_{a^{\prime}} \left|\hat{\pi}(a^{\prime}\mid x^{\prime})\right| \cdot\sup_r\left|F_{\hat{R}(x, a)}(r)-F_{R(x, a)}(r)\right| \int d F_{\gamma Z\left(x^{\prime}, a^{\prime}\right)}(u-r)  \\
&=\left\|F_{\hat{R}(x, a)}(\cdot)-F_{R(x, a)}(\cdot)\right\|_{\infty}. 
\end{align*}
%Therefore, the first term is bounded by $2\alpha\left\|F_{\hat{R}(x, a)}(r)-F_{R(x, a)}(r)\right\|_{\infty}$. %Based on the fact that variance reduction accelerates policy convergence, for all $(x, a)\in \mathcal{D}$, there exists some $N$, when $n(x,a)\geq N$, $\max_{x'} D_{TV}(\hat{\pi}\left(\cdot\mid x^{\prime})\| \pi_{opt}(\cdot \mid x^{\prime})\right)= \alpha(n)\leq 1$. Especially, when the advantage function or value function in the policy gradient is estimated by QCM, the total variation could achieve a lower value $\beta(n)<\alpha(n)$ . 

The second term can be bounded as follows:
\begin{align*}
&\sum_{x',a'}\left(\hat{P}(x^{\prime} \mid x, a)-P(x^{\prime} \mid x, a)\right) \pi_{opt}(a^{\prime} \mid x^{\prime}) F_{\gamma Z(x^{\prime}, a^{\prime})+R(x, a)}(u)\\
&\leq\sum_{x^{\prime}}\left(\hat{P}(x^{\prime} \mid x, a)-P(x^{\prime} \mid x, a)\right) \sum_{a^{\prime}} \pi_{opt}(a^{\prime} \mid x^{\prime}) \\
&\leq\left\|\hat{P}(\cdot \mid x, a)-P(\cdot \mid x, a)\right\|_1 \cdot\left\|\sum_{a^{\prime}} \pi_{opt}(a^{\prime} \mid \cdot)\right\|_{\infty} \\
&=\left\|\hat{P}(\cdot \mid x, a)-P(\cdot \mid x, a)\right\|_1 .
\end{align*}

Next, we show the two norms can be bounded. By the Dvoretzky-Kiefer-Wolfowitz (DKW) inequality, the following inequality holds with probability at least $1-\delta / 2$, for all $(x, a) \in \mathcal{D}$,
\begin{align*}
\left\|F_{\hat{R}(x, a)}(\cdot)-F_{R(x, a)}(\cdot)\right\|_{\infty} \leq \sqrt{\frac{1}{2 n(x, a)} \log \frac{4|\mathcal{X} \| \mathcal{A}|}{\delta}}.
\end{align*}
By Hoeffding’s inequality and an $l_1$ concentration bound for multinomial distribution\footnote{ see \href{https://nanjiang.cs.illinois.edu/files/cs598/note3.pdf}{https://nanjiang.cs.illinois.edu/files/cs598/note3.pdf.}}, the following inequality holds with probability at least $1-\delta / 2$,
\begin{align*}
\max _{x, a}\left\|\hat{P}(\cdot \mid x, a)-P(\cdot \mid x, a)\right\|_1 \leq \sqrt{\frac{2|\mathcal{X}|}{n(x, a)} \log \frac{4|\mathcal{X} \| \mathcal{A}|}{\delta}}.
\end{align*}

Consequently, the claim follows from combining the two inequalities, 
\begin{align*}
    \left\|F_{\hat{\mathcal{T}}^{\hat{\pi}} \eta(x, a)}-F_{\mathcal{T}^{\pi_{opt}} \eta(x, a)}\right\|_{\infty}\leq 2\alpha + \sqrt{\frac{1+4|\mathcal{X}|}{n(x, a)} \log \frac{4|\mathcal{X}||\mathcal{A}|}{\delta}}.
\end{align*}
\end{proof}

\section{Cornish-Fisher Expansion}\label{intro_CFE}

The Cornish-Fisher Expansion \cite{Fisher1938} is an asymptotic expansion used to approximate the quantiles of a probability distribution based on its cumulants. To be more explicit, let $X^{*}$ be a non-gaussian variable with mean 0 and variance 1. Then, the Cornish-Fisher Expansion can be represented as a polynomial expansion:
\begin{align*}
F^{-1}_{X^{*}}(\tau)= \sum_{i=0}^{\infty}a_{i}(\Phi^{-1}(\tau))^{i},
\end{align*}
where the parameters $a_i$ depend on the cumulants of the  $X^{*}$ and $\Phi$ is the standard normal distribution function.  To use this expansion in practice, we need to truncate the series. According to \citet{Fisher1938}, the highest power of $i$ must be odd, and the fourth order ($i = 3$) approximation is commonly used in practice. The parameters for the fourth order expansion are  $a_2 = a_0 = \frac{\kappa_3}{6}$, $a_1=1+5(\frac{\kappa_3}{6})^2-3\frac{\kappa_4}{24}$ and $a_3 = \frac{\kappa_4}{24}-2(\frac{\kappa_3}{6})^2$, where $\kappa_i$ denotes $i$-th cumulant.  Therefore, the fourth order expansion is
\begin{align*}
F^{-1}_{X^{*}}(\tau)= -\frac{\kappa_3}{6}+(1+5(\frac{\kappa_3}{6})^2-3\frac{\kappa_4}{24})\Phi^{-1}(\tau)+\frac{\kappa_3}{6}(\Phi^{-1}(\tau))^{2}+(\frac{\kappa_4}{24}-2(\frac{\kappa_3}{6})^2)(\Phi^{-1}(\tau))^{3}+\cdots.
\end{align*}
Now, simply define the $X^{*}$ as the normalization of $X$, $X=\mu+\sigma X^{*}$, with mean $\mu$ and variance $\sigma^2$. $F^{-1}_{X}(\tau)$ can be approximated by
\begin{align*}
F^{-1}_{X}(\tau)= \mu + \sigma \left(-\frac{\kappa_3}{6\sigma^3}+(1+5(\frac{\kappa_3}{6\sigma^3})^2-3\frac{\kappa_4}{24\sigma^4})\Phi^{-1}(\tau)+\frac{\kappa_3}{6\sigma^3}(\Phi^{-1}(\tau))^{2}+(\frac{\kappa_4}{24\sigma^4}-2(\frac{\kappa_3}{6\sigma^3})^2)(\Phi^{-1}(\tau))^{3}+\cdots\right).
\end{align*}
Denote skewness $s=\frac{\kappa_3}{\sigma^3}$, kurtosis $k=\frac{\kappa_4}{\sigma^4}$ and normal distribution quantile $z_{\tau}=\Phi^{-1}(\tau)$. Then, we can rewrite the above equation
\begin{align}\label{CFE_4}
F^{-1}_{X}(\tau)= \mu + \sigma\left( z_{\tau}+ (z_{\tau}^{2}-1) \frac{s}{6} + (z_{\tau}^{3}-2z_{\tau})\frac{k}{24} + (-2z_{\tau}^{3}+5z_{\tau})(\frac{s}{6})^2 +\cdots\right).
\end{align}

\subsection{Regression model selection}\label{model selection}

We use the R-Squared ($R^2$) statistic to determine the number of terms in \cref{CFE_4} that should be included in the regression model. $R^2$, also known as the coefficient of determination, is a statistical measure that shows how well the independent variables explain the variance in the dependent variable. In other words, it is a measure of how well the data fit the regression model.

\begin{figure}[!h]
\centering
\includegraphics[width = 1\linewidth]{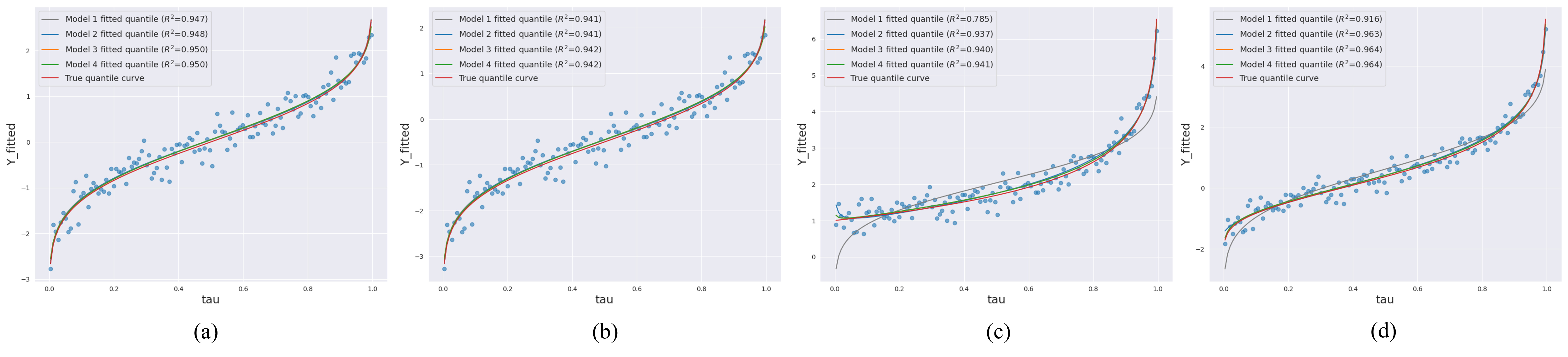}
\caption{Fitted quantile plot. (a) Normal, $\mathcal{N}(0,1)$. (b) Mixture Gaussian, $0.7\mathcal{N}(-2,1)+0.3\mathcal{N}(3,1)$. (c) Exponential, $Exp(1)=e^{-x}$. (d) Gumbel, $G(0,1)=e^{-(x+e^{-x})}$.}
\label{regression R2}
%\vskip -0.5in
\end{figure}

Consider the linear regression model,
\begin{align*}
\boldsymbol{\hat{Y}}=\mathbf{X}_{i} \boldsymbol{\beta}_{i}+\mathcal{E}.
\end{align*}
The dependent variable $\boldsymbol{Y}=(F^{-1}_{X}(\tau_1),\dots,F^{-1}_{X}(\tau_N))^{T}$ is composed of the quantiles from distribution of $X$, and $\mathcal{E}$ is the noise vector sampled from $\mathcal{N}(0,0.25)$. When the design matrix $\mathbf{X}_{1} = (1,\cdots,1)'$,  this regression model reduces to a one-sample problem, and $\boldsymbol{\beta}_{1}$ can be directly estimated by $\frac{1}{N}\sum_{n=1}^{N} F^{-1}_{X}(\tau_n)$. We then investigate the following four types of regression models, 

Model 1: 
$$
\mathbf{X}_{2}=\left(\begin{array}{cccc}
1, &  \cdots &,1  \\
z_{\tau_1},&\cdots &,z_{\tau_N} \\
\end{array}\right)^{T}, \boldsymbol{\beta}_{2}=\left(\mu,\sigma\right)^{T},
$$
Model 2: 
$$
\mathbf{X}_{3}=\left(\begin{array}{cccc}
1, &  \cdots &,1  \\
z_{\tau_1},&\cdots &,z_{\tau_N} \\
z_{\tau_1}^{2}-1,&\cdots &,z_{\tau_N}^{2}-1 \\
\end{array}\right)^{T}, \boldsymbol{\beta}_{3}=\left(\mu,\sigma,\sigma\frac{s}{6}\right)^{T},
$$
Model 3: 
$$
\mathbf{X}_{4}=\left(\begin{array}{cccc}
1 ,&  \cdots &,1  \\
z_{\tau_1},&\cdots &,z_{\tau_N} \\
z_{\tau_1}^{2}-1,&\cdots &,z_{\tau_N}^{2}-1 \\
z_{\tau_1}^{3}-3z_{\tau_1},&\cdots &,z_{\tau_N}^{3}-3z_{\tau_N} \\
\end{array}\right)^{T}, \boldsymbol{\beta}_{4}=\left(\mu,\sigma,\sigma\frac{s}{6},\sigma\frac{k}{24}\right)^{T},
$$
Model 4: 
$$
\mathbf{X}_{5}=\left(\begin{array}{cccc}
1 ,&  \cdots &,1  \\
z_{\tau_1},&\cdots &,z_{\tau_N} \\
z_{\tau_1}^{2}-1,&\cdots &,z_{\tau_N}^{2}-1 \\
z_{\tau_1}^{3}-3z_{\tau_1},&\cdots &,z_{\tau_N}^{3}-3z_{\tau_N} \\
-2z_{\tau_1}^{3}+5z_{\tau_1},&\cdots &,-2z_{\tau_N}^{3}+5z_{\tau_N} \\
\end{array}\right)^{T}, \boldsymbol{\beta}_{5}=\left(\mu,\sigma,\sigma\frac{s}{6},\sigma\frac{k}{24},\sigma(\frac{s}{6})^2\right)^{T}.
$$
\cref{regression R2} shows that the regression fitted values and corresponding $R^2$ across several distributions of $X$. As the number of independent variables increases, more variance in the error can be explained. However, having too many independent variables increases the risk of multicollinearity and overfitting. Based on practical considerations, we choose Model 3 as our regression model due to its satisfactory level of explainability. In the subsequent section, we will give a more in-depth interpretation of this regression model.

\subsection{Interpretation of the remaining term $\omega(\tau)$}\label{remaining term}
In this section, we explore the role of the remaining term $\omega(\tau)$ in the context of random design regression. As discussed in \cref{sec4}, we present a decomposition of the estimate $\hat{q}(\tau)$ of the $\tau$-th quantile, which includes contributions from the mean, noise error, and misspecified error. Specifically, we expressed the estimate as follows:
 \begin{align*}
\hat{q}(\tau) = & \mu + \omega_1(\tau) + \varepsilon(\tau).
\end{align*}
 where $\mu$ can be estimated using the mean estimator $\frac{1}{N}\sum q(\tau_i)$, which is commonly used in QDRL and IQN settings. However, this simple model fails to capture important information in the $\omega_1(\tau)$. To address this limitation, we employ the Cornish-Fisher Expansion to expand the equation, resulting in the following expression:
\begin{align*}
\hat{q}(\tau) = & \mu + z_{\tau} \sigma + \sigma \omega_2(\tau) + \varepsilon(\tau),\\
\hat{q}(\tau) = & \mu + z_{\tau} \sigma +  (z_{\tau}^2 -1)\sigma \frac{s}{6} + \sigma \omega_3(\tau) +\varepsilon(\tau),\\
\cdots
\end{align*}
where $\mu$ can be estimated by linear regression estimator given multiple quantile levels $\{\tau_i\}$, which can be sampled from a uniform distribution or predefined to be evenly spaced in $(0,1)$. In theory, higher-order expansions can capture more misspecified information in $\omega(\tau)$, leading to a more accurate representation of the quantile. However, as discussed before, expansions are typically limited to the fourth order in practice to balance the trade-off between model complexity and estimation accuracy. 

To gain a better understanding of the remaining term $\omega(\tau)$ and its impact on the regression estimator, consider the linear model,
\begin{align*}
\hat{q}(\tau) =  \mathbf{x_{\tau}^{\prime}}\beta  + \underbrace{\omega_\tau}_{\text{Misspecified error}} + \underbrace{\varepsilon}_{\text{Noise error}},
\end{align*}
where $\tau$ can be generally considered a uniform, $\mathbf{x_{\tau}}=(1,z_{\tau},z_{\tau}^{2}-1,...)^{\prime}\in\mathbb{R}^{d}$, and $\beta=(\mu,\sigma,\sigma\frac{s}{6},...)^{\prime} \in \mathbb{R}^{d}$. In particular, define the random variables,
$$
\varepsilon:=\hat{q}(\tau)-\mathbb{E}[\hat{q}(\tau) \mid \mathbf{x_{\tau}}] \quad \text { and } \quad \omega_\tau:=\mathbb{E}[\hat{q}(\tau) \mid \mathbf{x_{\tau}}]-\mathbf{x_{\tau}}^{\prime}\beta,
$$
where $\varepsilon$ corresponds to the noise with zero mean, $\sigma^2_{\text{noise }}$ variance and independent across different level of $\tau$, and $\omega_\tau$ corresponds to the misspecified error of $\beta$. Under the following conditions, we can derive a bound for the regression estimator in the misspecified model.

\textbf{Condition 1 (Subgaussian noise).} There exist a finite constant $\sigma_{\text {noise }} \geq 0$ such that for all $\lambda \in \mathbb{R}$, almost surely:
$$
\mathbb{E}[\exp (\lambda \varepsilon) \mid \mathbf{x_{\tau}}] \leq \exp \left(\lambda^2 \sigma_{\text {noise }}^2 / 2\right) .
$$

\textbf{Condition 2 (Bounded approximation error).} There exist a finite constant $C_{\text {bias }} \geq 0$, almost surely:
$$
\left\|\Sigma^{-1 / 2} \mathbf{x_{\tau}} \omega_\tau\right\|_{2} \leq C_{\text {bias }} \sqrt{d},
$$
where $\Sigma = \mathbb{E}[\mathbf{x_{\tau}} \mathbf{x_{\tau}}^{\prime}]$.

\textbf{Condition 3 (Subgaussian projections).} There exists a finite constant $\rho \geq 1$ such that:
$$
\mathbb{E}\left[\exp (\alpha^{\top} \Sigma^{-1 / 2} \mathbf{x_{\tau}})\right] \leq \exp \left(\rho \cdot\|\alpha\|^2_2 / 2\right), \quad \forall \alpha \in \mathbb{R}^d.
$$

\begin{theorem}\label{randomdesign}
Suppose that Conditions 1, 2, and 3 hold. Then for any $\delta \in(0,1)$ and with probability at least $1-3 \delta$, the following holds:

\begin{align*}
\left\|\hat{\beta}_{\mathrm{ols}}-\beta\right\|_{\Sigma}^2 
& \leq \underbrace{K_{\rho,\delta, N}^2\left(\frac{4  \mathbb{E}\left\|\Sigma^{-1 / 2} \mathbf{x_{\tau}} \omega_\tau\right\|^2_2 (1+8 \log (1 / \delta))}{N}+\frac{3 C_{\text {bias }}^2 d \log ^2(1 / \delta)}{N^2}\right)}_{\text{Misspecified error contribution}} \\
& + \underbrace{K_{\rho,\delta, N} \cdot \frac{\sigma_{\text {noise }}^2 \cdot(d+2 \sqrt{d \log (1 / \delta)}+2 \log (1 / \delta))}{N}}_{\text{Noise error contribution}},
\end{align*}
where $K_{\rho,\delta, N}$ is a constant depending on $\rho$, $\delta$ and $N$.
\end{theorem}
%&\leq 2\left\|\frac{1}{N}\sum_{i=1}^{N}\hat{\Sigma}^{-1} \mathbf{x_{\tau_i}} \omega_{\tau_{i}}\right\|_{\Sigma}^2+   2\left\| \frac{1}{N}\sum_{i=1}^{N} \hat{\Sigma}^{-1} \mathbf{x_{\tau_i}} \varepsilon\right\|_{\Sigma}^2 \\
\begin{proof}
    The proof of the above theorem can be easily adapted from Theorem 2 in \citet{RandomDesign}.
\end{proof}

The first term on the right-hand side represents the error due to model misspecification, which occurs when the true model differs from the assumed model. Intuitively, incorporating more relevant  information in $\omega(\tau)$ into explanation variables could decrease the quantity of 
$\mathbb{E}\left\|\Sigma^{-1 / 2} \mathbf{x_{\tau}} \omega_\tau\right\|^2_2$ and $C_{bias}$. Therefore, the accuracy of the estimator may be potentially improved by reducing the magnitude of the misspecified error. The second term represents the noise error contribution, which is inevitable and can only be controlled by increasing the sample size $N$.

\section{Experimental Details}\label{implement}
\subsection{Tabular experiment}
The parameter settings used for tabular control are presented in \cref{table1}. In the QEMRL case, the weight matrix $V$ is set as shown in the table based on domain knowledge indicating that the distribution has low probability support around its median. The greedy parameter decreases exponentially every 100 steps, and the learning rate decrease in segments every 50K steps.

\begin{table}[!h]
%\vskip 0.15in
\begin{center}
\caption{The (hyper-)parameters of QEMRL and QDRL used in the tabular control experiment.}
\begin{tabular}{lcccr}
\toprule
Hyperparameter & Value\\
\midrule
Learning rate schedule    & \{0.05,0.025,0.0125\}\\
Discount factor         & 0.999 \\
Quantile initialization & $\mathrm{Unif}(-0.5,0.5)$\\
Number of quantiles     &  128 \\
Number of training steps     & 150K \\
$\epsilon$-greedy  schedule  &       $0.9^{\lfloor t/100\rfloor}$   \\
Number of MC rollouts     &  10000      \\
Weight matrix $V$ (QEMRL only)   &  $diag\{1,1,\cdots,\underbrace{1.5,\cdots,1.5}_{\tau \in [0.45,0.55]},\cdots,1,1 \}$\\
\bottomrule
\label{table1}
\end{tabular}
\end{center}
%\vskip -0.1in
\end{table}

\subsection{Atari experiment}
We extend QEMRL to a DQN-like architecture, and we use the same architecture as QR-DQN, which we refer to as QEM-DQN \footnote{Code is available at \href{https://github.com/Kuangqi927/QEM}{ https://github.com/Kuangqi927/QEM}}. Our hyperparameter settings (\cref{table2}) are aligned with \citet{QRDQN} for a fair comparison. Additionally, we extend QEMRL to the unfixed quantile fraction algorithm IQN, which embeds quantile fraction $\tau$ into the quantile value network on the top of QR-DQN. In Atari, it is infeasible to determine the low probability supports for every state-action pair, therefore we only consider the heteroskedasticity that occurs in the tail and treat $V$ as a tuning parameter to select an appropriate value. For exploration experiments, we follow the settings of \citet{DLTV} and set the decay factor $c_{t} = c\sqrt{\frac{\text{log} t}{t}} $, where $c=50$.

\begin{table}[!h]
%\vskip 0.15in
\begin{center}
\caption{The hyperparameters of QEM-DQN and QR-DQN used in the Atari experiments.}
\begin{tabular}{lcccr}
\toprule
Hyperparameter & Value\\
\midrule
Learning rate    & 0.00005\\
Discount factor         & 0.99 \\
Optimizer & Adam\\
Bath size   &  32 \\
Number of quantiles     & 200 \\
Number of quantiles (IQN)     & 32 \\
Weight matrix $V$ (QEM-DQN only)   &  $diag\{\underbrace{1.5,\cdots,1.5}_{\tau \in [0.9,1)},\cdots,1,1,\cdots, \underbrace{1.5,\cdots,1.5}_{\tau \in (0,0.1]}\}$\\
\bottomrule
\label{table2}
\end{tabular}
\end{center}
%\vskip -0.1in
\end{table}

\subsection{MuJoCo experiment}
We extend QEMRL to a SAC-like architecture, and we use the same architecture of DSAC, named QEM-DSAC. Similarly, we extend QEMRL to an IQN version of DSAC. Hyperparameters and environment-specific parameters are listed in \cref{table3}. In addition, SAC has a variant that introduces a mechanism of fine-tuning $\alpha$ to achieve target entropy adaptively. While this adaptive mechanism performs well, we follow the use of fixed $\alpha$ suggested in the original SAC paper to reduce irrelevant factors.

\begin{table}[!h]
%\vskip 0.15in
\begin{center}
%\begin{small}
%\begin{sc}
\caption{The hyperparameters of QEM-DSAC and DSAC used in the MuJoCo experiments.}
\begin{tabular}{lcccr}
\toprule
Hyperparameter & Value\\
\midrule
Policy network learning rate & $0.0003$\\
Quantile Value network learning rate & $0.0003$\\
Discount factor         & 0.99 \\
Optimization & Adam \\
Target smoothing & $0.005$\\
Batch size & 256  \\
Minimum steps before training & $10000$      \\
Number of quantiles     & 32 \\
Quantile fraction embedding size (IQN) & 64  \\
Weight matrix $V$ (QEM-DSAC only)   &  $diag\{\underbrace{1.2,\cdots,1.2}_{\tau \in [0.9,1)},\cdots,1,1,\cdots, \underbrace{1.2,\cdots,1.2}_{\tau \in (0,0.1]}\}$\\
\bottomrule
\label{table3}
\end{tabular}

\begin{tabular}{lcccr}
\toprule
Environment &  Temperature Parameter\\
\midrule
Ant-v2 & 0.2\\
HalfCheetah-v2 & 0.2\\
Hopper-v2 & 0.2\\
Walker2d-v2 & 0.2\\
Swimmer-v2 & 0.2\\
Humanoid-v2 & 0.05\\
\bottomrule

\end{tabular}

%\end{sc}
%\end{small}
\end{center}
%\vskip -0.1in
\end{table}

\section{Additional Experimental Results}\label{additional experiments}

\subsection{Variance reduction for IQN}
%The sufficient condition $z_i = −z_{N −i}$, where $z_{i}=\Phi^{-1}(\tau_i)$, for variance reduction does not hold in IQN as $\tau$ is sampled from a uniform distribution. 
IQN does not satisfy the sufficient condition $z_{\tau_i} = -z_{\tau_{N-i}}$ since $\tau$ is sampled from a uniform distribution, rather than evenly spaced as in QDRL. To examine the impact of this on the inequality $(\frac{\sum_i v_i}{N})^2-1-1/(\frac{(\sum_{i}v_{i}\sum_{i}v_{i}z_{i}^2)}{(\sum_{i}v_{i}z_{i})^2}-1)>0$ in \cref{propostion1},  simulation experiments are conducted. We use the function $f(v_1,\cdots,v_N)=(\frac{\sum_i v_i}{N})^2-1-1/(\frac{(\sum_{i}v_{i}\sum_{i}v_{i}z_{\tau_i}^2)}{(\sum_{i}v_{i}z_{\tau_i})^2}-1)$ to examine this inequality, where $v_i>1$ and $\tau_i$ are sampled uniformly. In every trial,  $v_i$ are randomly sampled from $[1,M]$, repeating the process 100,000 times. The minimum values of $f(v_1,\cdots,v_N)$ are shown in the following \cref{table4} for varying values of $N$ and $M$.  The results indicate that the minimum of $f$ is always greater than 0, which demonstrates that the inequality holds in practice.

\begin{table}[!h]
\begin{center}
\caption{Minimum of $f$.}
\begin{tabular}{lcccr}
\toprule
Minimum of $f$ &  M &  $N$\\
\midrule
0.614 & 2 &  32\\
4.778 & 5&  32\\
43.143 & 20&  32\\
0.932 & 2 &  128\\
7.707 & 5 &  128\\
76.489 & 20 &  128\\
1.082 & 2 &  500\\
9.357 & 5 &  500\\
96.473 & 20 &  500\\
\bottomrule
\label{table4}
\end{tabular}
\end{center}
\end{table}

\subsection{Weight $V$ tuning experiments}
\begin{figure}[!ht]
\centering
\includegraphics[width = 0.8\linewidth]{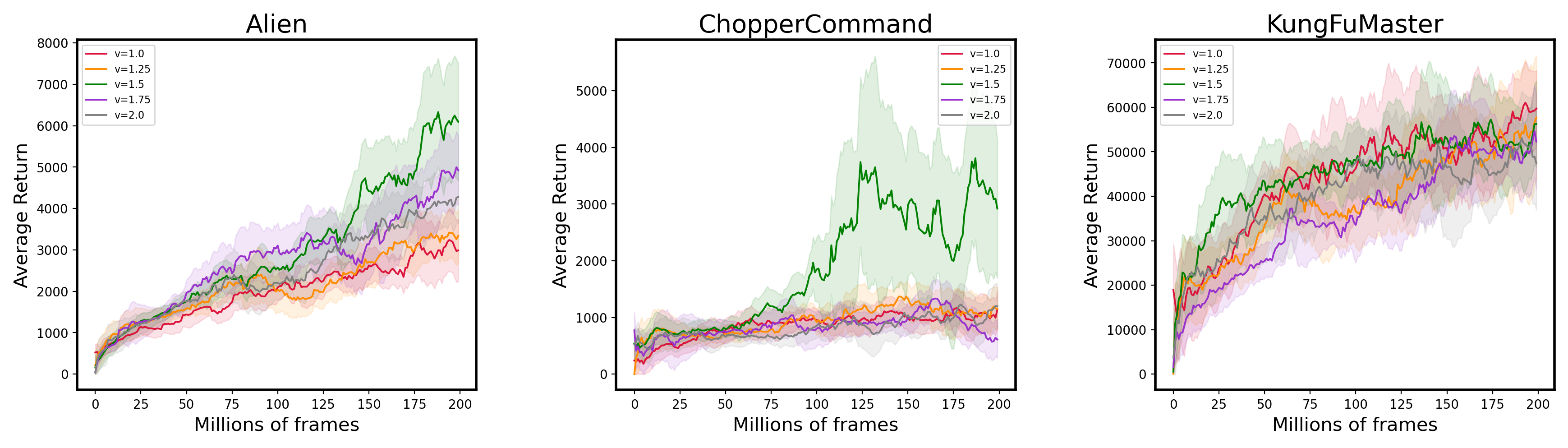}
\includegraphics[width = 0.8\linewidth]{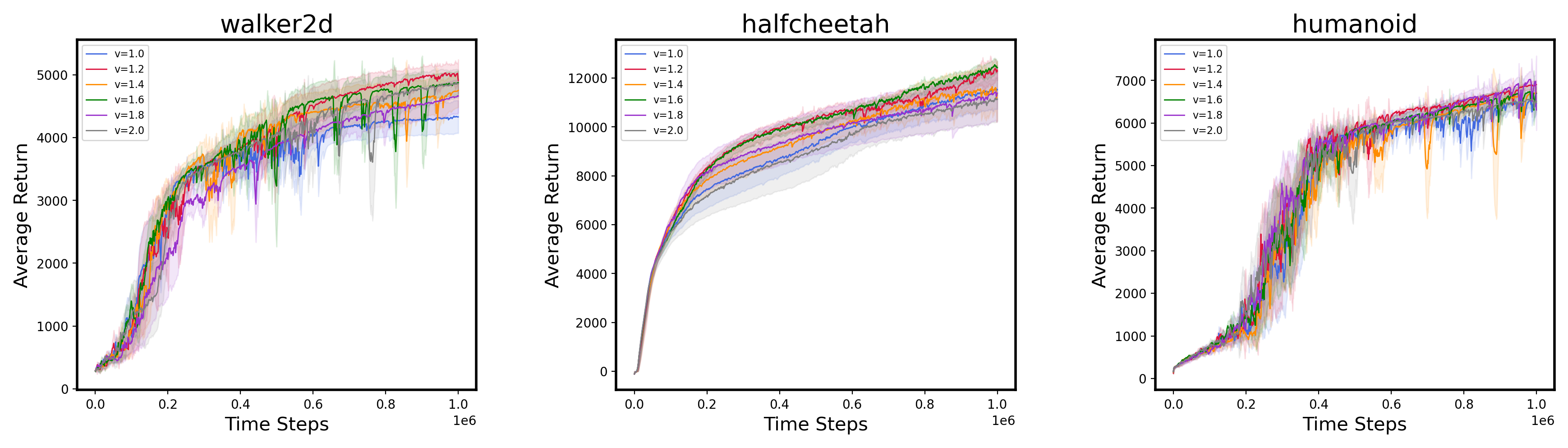}
\caption{ Comparison of different weight $v$ in QEM-DSAC and QEM-DQN experiments}
%\vskip -0.2in
\end{figure}

\subsection{Additional Atari results}
\begin{figure}[!ht]
\centering
\includegraphics[width = 0.8\linewidth]{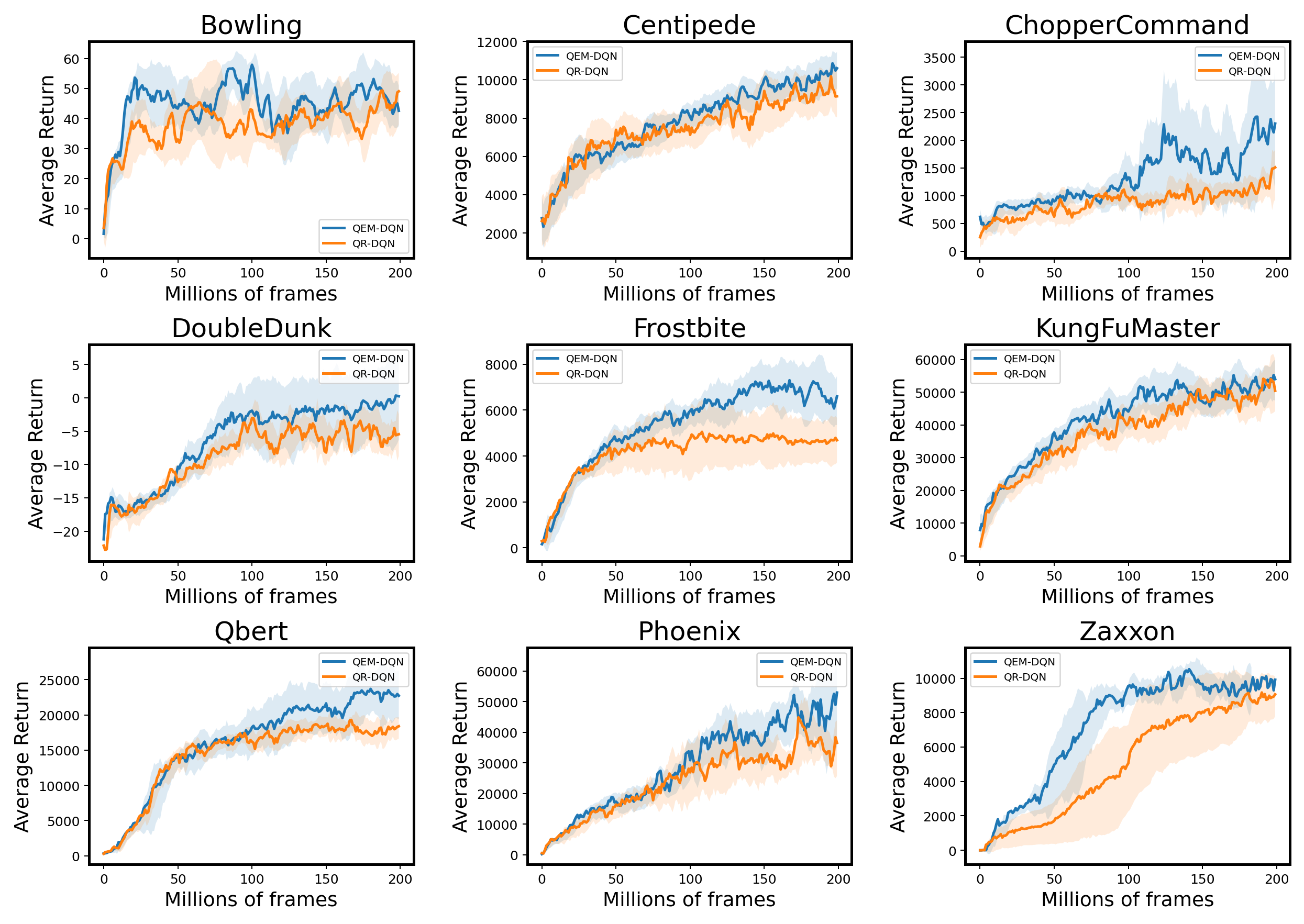}
\caption{ Comparison of QEM-DQN and QR-DQN across 9 Atari games}
%\vskip -0.2in
\end{figure}

\begin{figure}[!ht]
\centering
\includegraphics[width = 0.8\linewidth]{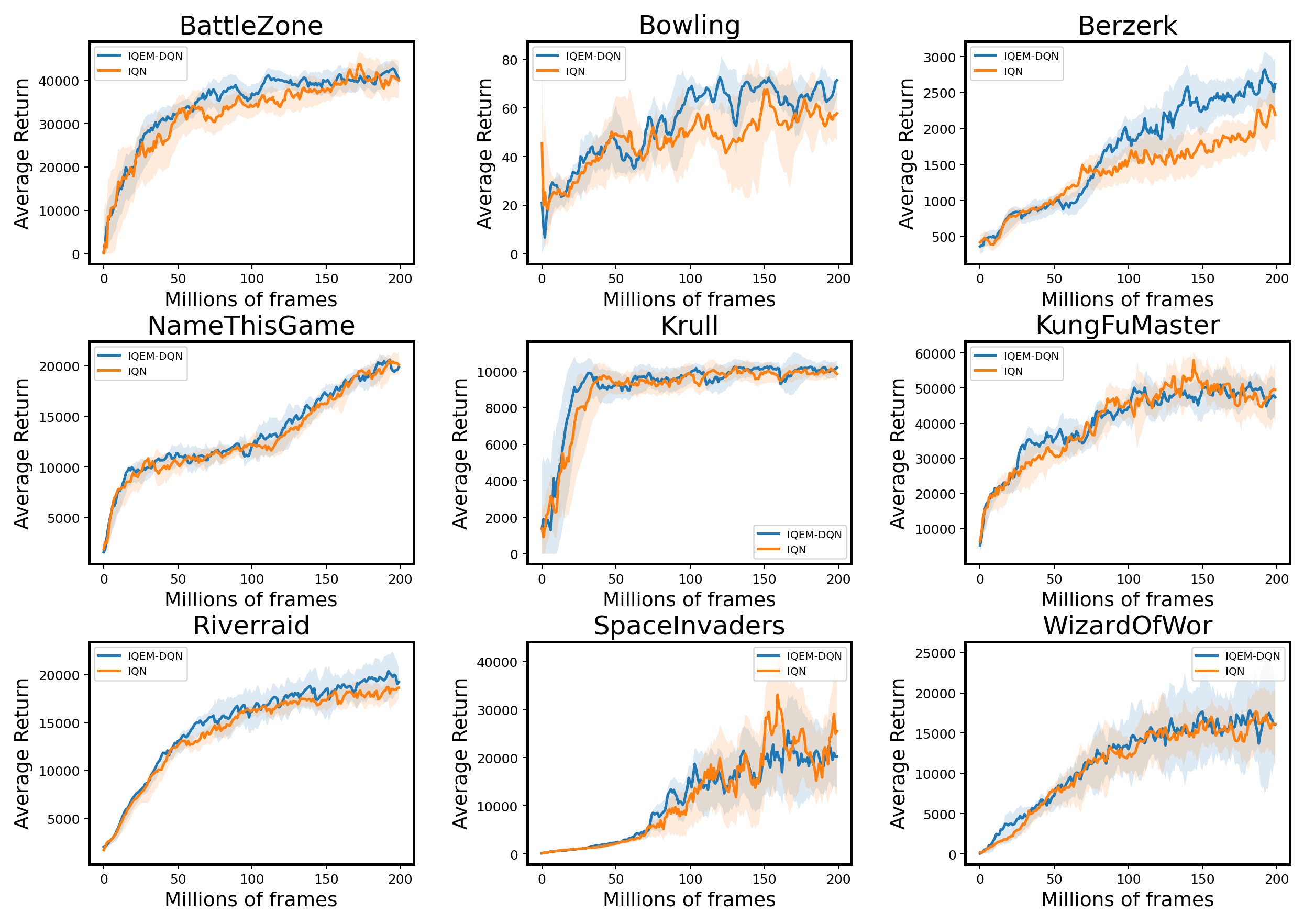}
\caption{ Comparison of IQEM-DQN and IQN across 9 Atari games}
%\vskip -0.2in
\end{figure}

\begin{figure}[!ht]
\centering
\includegraphics[width = 0.8\linewidth]{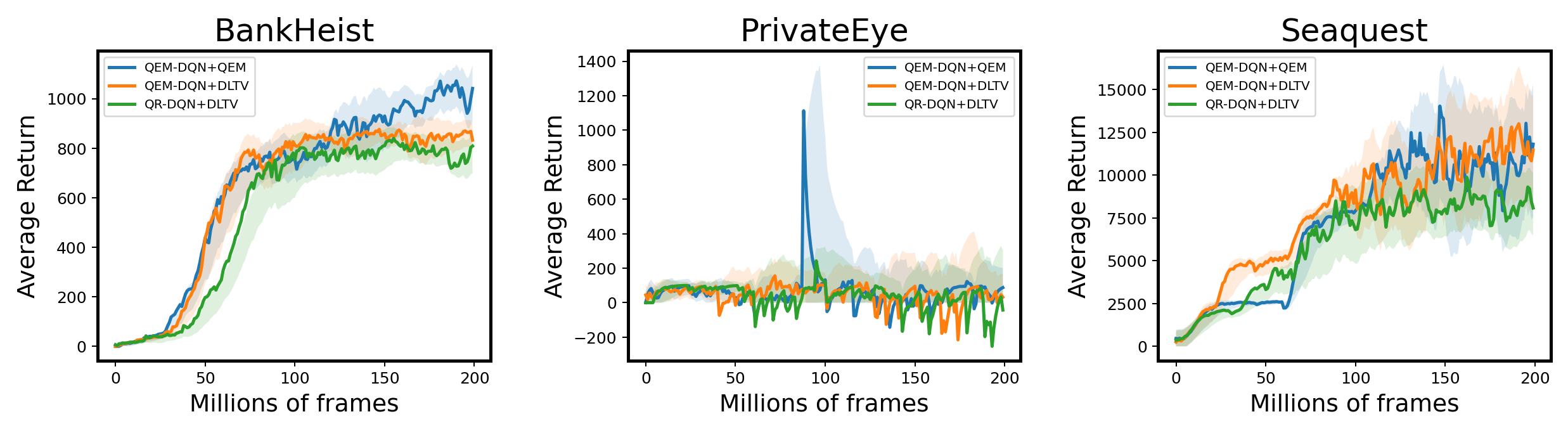}
\caption{ Comparison of QEM and DLTV across 3  hard-explored Atari games}
%\vskip -0.2in
\end{figure}

\end{document}